\documentclass[11pt]{article}
\usepackage{microtype}

\usepackage{amsmath,amssymb}
\usepackage{amsfonts,bm,bbm} 
\usepackage{amssymb,amsthm,enumitem}
\usepackage{tikz}
\usetikzlibrary{positioning,arrows.meta,fit,calc}
\usetikzlibrary{decorations.pathmorphing,positioning,patterns,snakes}
\usepackage{multirow}
\usepackage{xcolor,graphicx,float,booktabs}
\usepackage[most]{tcolorbox}
\usepackage{fvextra}
\usepackage{setspace}
\usetikzlibrary{shadows, backgrounds}

\tcbuselibrary{listings,breakable,skins}
\usepackage{tabularx}
\usepackage{makecell}
\usepackage{booktabs}
\newcolumntype{L}[1]{>{\raggedright\arraybackslash}m{#1}}
\newcolumntype{C}[1]{>{\centering\arraybackslash}m{#1}}
\newcolumntype{R}[1]{>{\raggedleft\arraybackslash}m{#1}}

\usepackage{epstopdf}

\usepackage{verbatim} 
\allowdisplaybreaks[4]
\usepackage{authblk}
\usepackage{cite}
\usepackage{hyperref}
\hypersetup{hypertex=true,
colorlinks=true,
linkcolor=blue,
anchorcolor=blue,
citecolor=blue}

\usepackage[round]{natbib}
\usepackage[
margin=1in,
includefoot,
footskip=30pt,
]{geometry}

\usepackage[algo2e,ruled,vlined]{algorithm2e}
\usepackage{subfig}
\usepackage{graphicx}

\newcommand{\E}{\mathbb{E}}

\newcommand{\1}{\mathbbm{1}}

\newtheorem{theorem}{Theorem}[section]

\numberwithin{equation}{section}
\definecolor{codeblue}{rgb}{0.1,0.2,0.6}
\definecolor{codegreen}{rgb}{0.0,0.5,0.0}
\definecolor{codered}{rgb}{0.6,0.1,0.1}
\definecolor{codegray}{rgb}{0.45,0.45,0.45}
\definecolor{codebg}{rgb}{0.97,0.97,0.97}

\lstdefinestyle{pythonstyle}{
  language=Python,
  basicstyle=\ttfamily\small,
  backgroundcolor=\color{codebg},
  keywordstyle=\color{codeblue}\bfseries,
  stringstyle=\color{codered},
  commentstyle=\color{codegreen}\itshape,
  numberstyle=\tiny\color{codegray},
  breaklines=true,
  breakatwhitespace=true,
  showstringspaces=false,
  frame=single,
  rulecolor=\color{black!25},
  tabsize=2,
  columns=fullflexible,
  keepspaces=true,
  xleftmargin=1.2em
}
\definecolor{myblue}{RGB}{55,102,170}
\definecolor{myred}{RGB}{198,40,40}
\definecolor{mygreen}{RGB}{35,130,70}
\definecolor{myorange}{RGB}{220,120,20}
\definecolor{lightblue}{RGB}{240,246,255}
\definecolor{lightred}{RGB}{253,239,239}
\definecolor{lightgreen}{RGB}{238,248,240}
\definecolor{lightorange}{RGB}{255,247,235}

\newcommand{\badhl}[1]{\colorbox{lightred}{\textcolor{myred}{\textbf{#1}}}}
\newcommand{\goodhl}[1]{\colorbox{lightgreen}{\textcolor{mygreen}{\textbf{#1}}}}
\newcommand{\key}[1]{{\color{myblue}\textbf{#1}}}

\newtcolorbox{nicebox}[3][]{
  breakable,
  colback=#2,
  colframe=myblue,
  colbacktitle=myblue,
  coltitle=white,
  title={#3},
  fonttitle=\bfseries,
  boxrule=0.8pt,
  arc=2mm,
  left=2mm,
  right=2mm,
  top=1mm,
  bottom=1mm,
  #1
}

\title{Uncertainty-Aware Simulation-Based Inference for Operations Research with Large Language Models}
\author[1]{Guo Liang}
\author[2]{Shaochong Lin}
\author[3]{Zuo-Jun Max Shen}
\author[4]{Kun Zhang}

\affil[1]{Institute of Statistics and Big Data, Renmin University of China, Beijing 100872, China. liangguo000221@ruc.edu.cn}

\affil[2]{Department of Data and Systems Engineering, The University of Hong Kong, Hong Kong 999077, China. shaoclin@hku.hk}

\affil[3]{Faculty of Engineering \& Faculty of Business and Economics, The University of Hong Kong 999077, Hong Kong, China. maxshen@hku.hk}

\affil[4]{School of Information, Renmin University of China, Beijing 100872, China. kunzhang@ruc.edu.cn}

\date{}

\begin{document}
\normalsize
\maketitle

\begin{abstract}
	Deploying large language models (LLMs) for operations research (OR) tasks remains challenging because correctness depends on a coherent modeling process, not merely a correct final answer. Standard autoregressive generation operates on a myopic policy, which sometimes fails to anticipate whether a partial formulation can be validly extended into a globally consistent optimization model. Consequently, locally plausible steps may propagate into catastrophic downstream formulation or solver code errors. To address this, we propose an uncertainty-aware, training-free inference framework for OR mathematical modeling. Without updating model parameters, our method evaluates intermediate candidate steps using short lookahead simulations to quantify downstream predictive uncertainty or probability concentration. Candidates that demonstrate a higher likelihood of yielding coherent mathematical formulations are then dynamically selected via importance resampling. Empirical evaluations across multiple OR benchmarks (including NL4OPT, MAMO, and IndustryOR) demonstrate that our framework consistently outperforms both standard and low-temperature baselines, establishing an efficient, training-free paradigm for reliable OR formulation generation.
	
\end{abstract}

\section{Introduction}

Operations research (OR) translates complex decision problems into structured mathematical formulations consisting of variables, constraints and objectives, whose structural integrity directly governs the quality of high-stakes resource allocation, scheduling and planning decisions. In the past, this translation process has been hindered by the scarcity of expert modelers capable of reliably reconstructing coherent optimization models from natural language requirements. While the recent advent of large language models (LLMs) offers a promising avenue to automate this pipeline, their underlying generation paradigm remains inherently constrained by token-by-token\footnote{In a language model, a token is a text unit that the model generates at each step. Figure \ref{fig:autoregressive_generation} illustrates how tokens are generated in an LLM.} autoregression, which evaluates each subsequent step based solely on immediate local probabilities. Consequently, a sequence of individually plausible choices, such as introducing an isolated constraint or an unindexed variable, frequently compounds into a globally inconsistent formulation characterized by conflicting constraints or undefined decision domains. The fundamental structural difficulty lies in this misalignment that the correctness of a mathematical model is a global property of the entire formulation, whereas the generation policy that constructs it operates locally, meaning that local optimality provides no guarantee of global coherence.

When generation is unstable, a standard mitigation strategy is low-temperature sampling\footnote{Low-temperature sampling is also called low-temperature decoding in the LLM literature. In this paper, we employ ``sampling'' to guide the generation process of the LLM.}, which reduces the temperature parameter to sharpen the local distribution and induces more deterministic commitments to the most probable choices at each step \citep{chen2021evaluatinglargelanguagemodels}. This heuristic is well-motivated for tasks where correctness depends primarily on individual token accuracy, such as translation fluency and basic syntactic validation. In those domains, higher local probability often correlates with higher correctness, and suppressing low-probability alternatives effectively reduces random variation. However, in OR modeling, the dominant failure mode stems from a fundamentally different source. A generated formulation must satisfy global consistency requirements, meaning that variables must be defined before use, constraints must reference the correct index sets, and objective functions must align with decision domains. A high-probability token is not inherently more capable of preserving these properties downstream than a less probable alternative. When an early modeling error lies on a highly probable path, low-temperature sampling does not avoid the mistake, but rather arrives at it more deterministically. Consequently, temperature reduction addresses the wrong source of variation by controlling token-level statistical uncertainty while leaving formulation-level structural uncertainty completely untouched.


Unlike low-temperature sampling, which operates strictly at the token level, reinforcement learning with verifiable rewards (RLVR) evaluates the generation process against a trajectory-level objective \citep{guo2025deepseek,hu2025open}. By shifting probability mass toward candidate paths that optimize a global reward, this approach reweights intermediate choices according to their downstream potential rather than their immediate local probability. In this sense, reinforcement learning addresses the correct structural problem by evaluating partial formulations based on their long-horizon future consequences. However, deploying this paradigm for OR tasks encounters two substantial obstacles. The first is practical. Evaluating mathematical formulations typically requires dense solver integration to verify feasibility, optimality, and constraint satisfaction. This makes automated feedback computationally expensive, sparse, and difficult to design without introducing unintended incentives during iterative training \citep{rafailov2023direct}. The second obstacle is more fundamental. Recent evidence suggests that post-training reinforcement learning primarily functions by reallocating probability mass toward reasoning paths already available within the base model distribution, rather than injecting entirely new capabilities \citep{yue2025does,karan2025reasoning}. 
These practical and conceptual limitations underscore the need for a training-free alternative that can explicitly evaluate downstream formulation coherence at inference time \footnote{Note that in the context of LLMs, ``inference time'' refers to the operational phase where an LLM generates outputs from input prompts. This should be distinguished from ``inference'' in statistics, which typically involves drawing conclusions about population parameters or testing hypotheses based on observed data.} without requiring model retraining, external reward models, or parameter updates.


In this paper, we develop such an alternative through an uncertainty-aware inference framework that combines lookahead evaluation with importance resampling. Instead of updating model parameters, the framework shifts probability during generation by using short future simulations to assess the reliability of the current partial formulation. The intuition is that a well-aligned partial formulation should lead to future extensions that remain concentrated around compatible variables, constraints, and code structures, whereas a misaligned one is more likely to branch into inconsistent or unstable downstream formulations. We formalize this intuition through two lookahead-based reward functions: one based on downstream probability concentration and the other on predictive uncertainty. By estimating these rewards through multiple short rollouts from the base model and applying importance resampling, the framework assigns greater selection probability to candidates with higher structural reliability. In this way, it improves OR generation without additional supervision, external reward models, or gradient-based updates. Specifically, we make the following main contributions in this paper.

\begin{enumerate}
\item {\textbf{Examples of myopic policy failures in OR modeling.}} We establish that local plausibility does not guarantee the global structural integrity required for valid OR formulations. Through two complementary analyses, we demonstrate that locally probable modeling choices can propagate into globally invalid formulations and solver errors -- a risk that persists even under low-temperature sampling, which often increases the model's confidence in its early mistakes. By quantifying the path dependency of OR reasoning, we find that correct final answers almost exclusively originate from flawless intermediate reasoning chains with negligible recovery from early errors, a characteristic that necessitates a shift toward trajectory-level assessment.

\item {\textbf{Efficient inference-time framework that integrates multi-path rollouts with importance resampling.}} This approach introduces two reward functions (power reward and entropy reward) to assess partial formulations by their downstream probability concentration and predictive uncertainty, capturing the intuition that well-aligned modeling steps lead to more coherent future continuations. By integrating these rewards into an importance resampling procedure, the framework dynamically reallocates probability mass toward reliable reasoning paths during generation without requiring model parameter updates or external reward models, ensuring an inference overhead that remains practically scalable for complex deployment.

\item {\textbf{Empirical validation across diverse OR benchmarks.}} We provide extensive empirical evidence of the framework’s effectiveness across multiple representative OR benchmarks, including NL4OPT, MAMO, and IndustryOR. Our results demonstrate that the proposed method consistently outperforms standard sampling and competitive low-temperature baselines across various levels of problem complexity. Notably, the framework yields the most significant performance gains in highly constrained optimization scenarios, implying that uncertainty-aware trajectory evaluation is particularly vital for solving sophisticated, industrial-scale modeling tasks where structural consistency is paramount.
\end{enumerate}

The remainder of this paper is structured as follows. Section \ref{sec:review} reviews the related literature on applying LLMs to optimization modeling, uncertainty in text generation, and inference-time intervention methods. Section \ref{sec:background} introduces the background of LLM generation and motivates our approach by demonstrating the insufficiency of myopic sampling policies in highly path-dependent OR tasks. Section \ref{sec:method} formally presents our uncertainty-aware inference framework, detailing the lookahead-based reward design and the importance resampling procedure. Section \ref{sec:numerical} reports the numerical experiments and empirical results across multiple OR benchmarks. Finally, Section \ref{sec:conclusion} concludes the paper and discusses potential directions for future research. All mathematical proofs and additional generation examples are provided in the appendices.

\section{Literature Review}\label{sec:review}

\subsection{LLMs for Optimization Modeling}

The application of language models to OR has progressed from auto-formulation systems to more specialized agentic and training-based frameworks. Early work translates problem descriptions into optimization models and establishes NL4Opt as a benchmark for this task \citep{ramamonjison2022augmenting,ramamonjison2023nl4opt}. To handle more complex instances, OptiMUS decomposes modeling into formulation, solver-code generation, debugging, and testing \citep{ahmaditeshnizi2024optimus}, while Chain-of-Experts introduces specialized agents with planning and reflection \citep{xiao2024chain}. More recently, lightweight few-shot and tool-use frameworks have been developed for large-scale optimization modeling without task-specific fine-tuning \citep{liang2026leanllmopt}.

Another stream relies on benchmarks, synthetic data, and model training. The MAMO benchmark evaluates mathematical modeling ability with solver-verifiable instances \citep{huang2024mamo}, while the OptiBench benchmark measures optimization modeling over diverse problem types and introduces ReSocratic data synthesis for supervised fine-tuning \citep{yang2024optibench}. Along this line, the ORLM framework improves open-source models through semi-automated data synthesis and instruction tuning \citep{huang2025orlm}, and the LLMOPT framework further adds multi-instruction tuning and self-correction for general optimization problems \citep{jiang2025llmopt}. Recent test-time reinforcement learning methods further improve modeling and solving performance with fewer synthetic examples \citep{ding2026orr1}. Despite these advances, complex constraint construction remains a major bottleneck as task complexity increases \citep{chen2026optengine}. Unlike the above studies, which mainly improve OR modeling through agents or training, our method keeps the base model fixed and focuses on controlling intermediate formulation choices at inference time.

\subsection{Uncertainty in Language Model Generation}

Uncertainty estimation is widely used to assess the reliability of LLM outputs. Calibration studies show that language models can be poorly calibrated in question answering \citep{jiang2021calibration}. Larger models can sometimes estimate whether their own proposed answers are true, though this ability depends on task and prompting format \citep{kadavath2022language}. Beyond answer-level confidence, recent inference-time methods use uncertainty to allocate computation. For example, Entropy-Tree branches only at high-entropy reasoning positions \citep{wei2025entropy}, foresight-based decoding uses future simulations to evaluate candidate steps \citep{zhao2024phidecoding} and adaptive confidence sampling relies on answer stability to guide resampling and termination \citep{shi2025caws}.

Another related direction measures semantic rather than token-level uncertainty. Semantic entropy groups meaning-equivalent generations and improves correctness prediction \citep{kuhn2023semantic}. The same idea has also been used to detect hallucinations at scale \citep{farquhar2024detecting}. However, semantic uncertainty usually requires entailment models, LLM judges, or clustering over generated meanings. 

\subsection{Inference-Time Intervention and Resampling}

Inference-time intervention methods modify generation without changing model parameters. Guided and contrastive decoding intervene at the token-probability level. For example, the FUDGE method uses future discriminators to steer generation toward desired attributes \citep{yang2021fudge}, the DExperts method combines expert and anti-expert distributions through product-of-experts \citep{liu2021dexperts}, and the DoLa method contrasts logits from different layers to improve factuality \citep{chuang2023dola}. These methods mainly target controllability or factuality, while our setting requires maintaining mathematical consistency across a full OR formulation.

Other methods generate multiple reasoning paths and select or reweight outputs. The self-consistency method aggregates complete chains of thought by majority vote \citep{wang2023selfconsistency}, while the confidence-guided early stopping method reduces the required number of samples by stopping once posterior confidence is high \citep{aghazadeh2025cges}. More generally, the twisted sequential Monte Carlo framework provides a potential-based view of trajectory reweighting \citep{zhao2024probabilistic}. Process reward models instead evaluate intermediate reasoning steps, but typically require human-labeled process data or a trained verifier \citep{lightman2024lets}. Our framework is closest to process-level steering, but derives the steering signal from short lookahead rollouts of the frozen base model rather than from final-output aggregation or an external reward model.

\section{Background and Motivation}\label{sec:background}

\subsection{Background}
Let $\mathcal{X}$ be a finite vocabulary of tokens, where a token is the basic text unit used by the language model. Given an input prompt $c$, a generated response can be represented as a sequence of tokens $\mathbf{x}=x_{1:T}=(x_1,\ldots,x_T)\in\mathcal{X}^T$, where each $x_t\in\mathcal{X}$. Modern large language models (LLMs) generate such sequences by assigning a probability to each next token conditional on both the prompt and the tokens already generated. That is, they define an autoregressive distribution over $\mathcal{X}^T$ by the factorization $p(\mathbf{x} | c)=\prod_{t=1}^{T}p(x_t | c,x_{<t})$, where $x_{<t}=(x_1,...,x_{t-1})$ denotes the generated prefix before step $t$. Similar notation is used for $x_{>t}$, $x_{\leq t}$, and $x_{\geq t}$. The generation procedure of an LLM is shown in Figure \ref{fig:autoregressive_generation}. For open-source models, the conditional distribution $p(\cdot | c, x_{1:t-1})$ at each step is directly accessible. When the prompt $c$ is clear from context, we write $p(\mathbf{x})$ as shorthand for $p(\mathbf{x} | c)$. Throughout this paper, $p(\mathbf{x})$ denotes the base language model after supervised fine-tuning (SFT) on OR tasks. We assume that this model already captures substantial task-relevant knowledge and reasoning patterns, and our goal is to improve how such knowledge is utilized during generation.

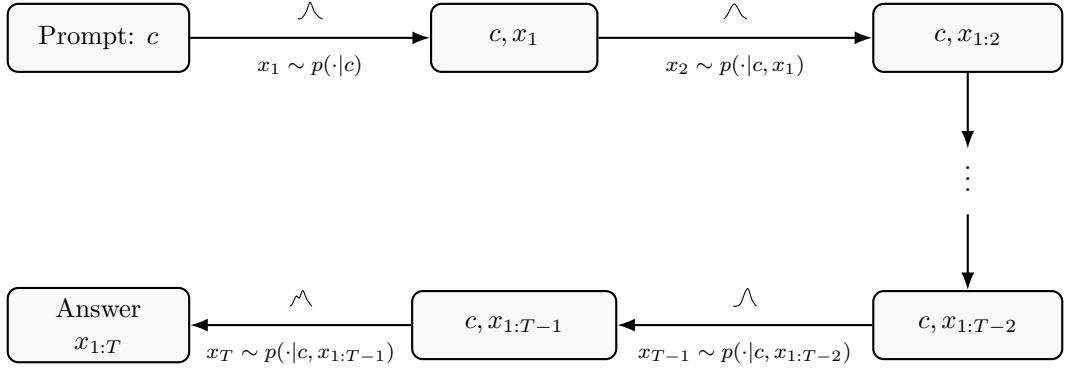
\begin{figure}[t]
\centering

\newcommand{\distA}{%
\begin{tikzpicture}[x=0.11cm,y=0.11cm,baseline=-0.45ex]
\draw[-,thin,smooth] plot coordinates {
(0,0) (0.5,0.1) (1.0,0.8) (1.5,2.0) (2.0,0.9) (2.6,0.2) (3.0,0)
};
\end{tikzpicture}}

\newcommand{\distB}{%
\begin{tikzpicture}[x=0.11cm,y=0.11cm,baseline=-0.45ex]
\draw[-,thin,smooth] plot coordinates {
(0,0) (0.4,0.4) (0.9,1.4) (1.3,1.8) (1.7,1.3) (2.2,0.4) (2.8,0)
};
\end{tikzpicture}}

\newcommand{\distC}{%
\begin{tikzpicture}[x=0.11cm,y=0.11cm,baseline=-0.45ex]
\draw[-,thin,smooth] plot coordinates {
(0,0) (0.6,0.05) (1.0,0.5) (1.4,1.6) (1.8,2.1) (2.3,0.7) (2.8,0)
};
\end{tikzpicture}}

\newcommand{\distD}{%
\begin{tikzpicture}[x=0.11cm,y=0.11cm,baseline=-0.45ex]
\draw[-,thin,smooth] plot coordinates {
(0,0) (0.3,0.2) (0.8,1.2) (1.2,1.0) (1.6,1.9) (2.1,0.8) (2.7,0)
};
\end{tikzpicture}}

\begin{tikzpicture}[
    >=Latex,
    font=\small,
    box/.style={
        draw,
        rounded corners,
        thick,
        minimum height=9mm,
        minimum width=24mm,
        align=center,
        fill=gray!5
    },
    flow/.style={->, thick},
    proc/.style={draw=blue, dashed, rounded corners, thick, inner sep=6mm},
    toplabel/.style={font=\scriptsize, align=center},
    bottomlabel/.style={font=\scriptsize, align=center}
]

\node[box, minimum width=24mm] (prompt) at (-0.5,0) {Prompt: $c$};
\node[box, minimum width=24mm] (answer) at (-0.5,-3.8) {Answer\\ $x_{1:T}$};

\node[box, minimum width=22mm] (s1) at (5.0,0) {$c,x_1$};
\node[box, minimum width=25mm] (s2) at (11.0,0) {$c,x_{1:2}$};

\node[box, minimum width=25mm] (s3) at (11.0,-3.8) {$c,x_{1:T-2}$};
\node[box, minimum width=27mm] (s4) at (5.0,-3.8) {$c,x_{1:T-1}$};

\draw[flow] (prompt.east) -- (s1.west)
node[midway, above=0.6mm, toplabel] {\distA}
node[midway, below=0.8mm, bottomlabel] {$x_1 \sim p(\cdot | c)$};

\draw[flow] (s1.east) -- (s2.west)
node[midway, above=0.6mm, toplabel] {\distB}
node[midway, below=0.8mm, bottomlabel] {$x_2 \sim p(\cdot | c,x_1)$};

\draw[flow] (s2.south) -- ++(0,-1.0);
\draw[flow] ($(s3.north)+(0,1.0)$) -- (s3.north);
\node[font=\small] at ($(s2.south)!0.45!(s3.north)$) {$\vdots$};

\draw[flow] (s3.west) -- (s4.east)
node[midway, above=0.6mm, toplabel] {\distC}
node[midway, below=0.8mm, bottomlabel] {$x_{T-1} \sim p(\cdot | c,x_{1:T-2})$};

\draw[flow] (s4.west) -- (answer.east)
node[midway, above=0.6mm, toplabel] {\distD}
node[midway, below=0.8mm, bottomlabel] {$x_T \sim p(\cdot | c,x_{1:T-1})$};

\end{tikzpicture}

\caption{How an LLM generates text step by step. Starting from the input question (prompt) $c$, the model repeatedly predicts a distribution over the next possible text unit (token), samples one, appends it to the current prefix (specifically, at step $t$, prefix is $c,x_{1:t-1}$), and continues until the full answer $x_{1:T}$ is produced.}
\label{fig:autoregressive_generation}
\end{figure}

In practice, a common way to stabilize LLM reasoning is low-temperature sampling. At each step, the base model assigns a probability distribution $p(\cdot | x_{<t})$ over possible next tokens. Temperature sampling rescales this distribution as
\begin{align*}
	p_\tau(x_t | x_{<t}) = \frac{p(x_t | x_{<t})^{1/\tau}}{\sum_{a\in\mathcal{X}} p(a | x_{<t})^{1/\tau}},
\end{align*}
where $\tau > 0$ denotes the temperature parameter. When $\tau=1$, this recovers the original model distribution. When $\tau<1$, the distribution becomes more concentrated on tokens that already have high probability under the base model, thereby reducing randomness in generation.

\subsection{Insufficiency of Myopic Policy}\label{sec:myopic_insuff}

Although low-temperature sampling (which constitutes a myopic policy, as it relies solely on the local probability $p(x_t|x_{<t})$ while remaining agnostic to the downstream information at subsequent steps $t+1,...,T$) is widely adopted as a strong baseline, its inherent preference for high-probability tokens is poorly aligned with the requirements of OR problem solving. In OR tasks, correctness depends not on textual fluency but on mathematical consistency, implementation correctness, and solution quality. To understand why myopic policy is insufficient under these conditions, we combine a case study with a Pass@$k$ versus CoT-Pass@$k$ comparison. Together, these analyses reveal two compounding limitations.

First, the following case study demonstrates that a locally plausible modeling choice can later invalidate the global formulation. Consider an OR prompt in which the task is to construct a mathematical model and its corresponding Python implementation. Specifically, the problem is to determine the optimal quantities of six available food items (Steak, Tofu, Chicken, Broccoli, Rice, and Spinach) so as to meet given nutritional requirements (protein, carbohydrates, and calories) while minimizing total cost.

\begin{nicebox}{lightblue}{A Toy Example}
\small

\key{Prompt:} Below is an operations research question. Build a mathematical model and corresponding python code using `coptpy' that appropriately addresses the question.

\medskip
\key{Question:} Imagine you are a dietitian and you have been tasked with creating a meal plan for a bodybuilder. You have six food items to choose from: Steak, Tofu, Chicken, Broccoli, Rice, and Spinach. Each food provides certain amounts of protein, carbohydrates, and calories, and each has its own cost. Here's the nutritional value and cost of each food: Steak: It gives you 14 grams of protein, 23 grams of carbohydrates, and 63 calories. Tofu: It offers 2 grams of protein, 13 grams of carbohydrates, and 162 calories. Chicken: It packs a punch with 17 grams of protein, 13 grams of carbohydrates, and gives you 260 calories for. Broccoli: It provides 3 grams of protein, a mere 1 gram of carbohydrates, and 55 calories. Rice: It gives a hearty 15 grams of protein, 23 grams of carbohydrates, and 231 calories. Spinach: It provides 2 grams of protein, 8 grams of carbohydrates, and a huge 297 calories. 

{\textbf{Your goal is to ensure that the bodybuilder gets at least 83 grams of protein, 192 grams of carbohydrates, and 2089 calories from whatever combination of these foods you choose. The challenge is to keep the cost as low as possible while meeting these nutritional targets. What is the minimum cost to meet these nutritional requirements with the available food options?}}

\end{nicebox}

Under low-temperature sampling, a natural and highly probable early step is to define variables as nonnegative food quantities, which is an approach consistent with the standard linear programming formulation of diet problems. This tendency arises because base models have been extensively exposed to linear programming examples, leading them to favor continuous formulations as a default ``reasonable'' template. However, if the food quantities must represent discrete servings or units, the correct feasible set is $\mathbb{Z}_{+}^6$ rather than $\mathbb{R}_{+}^{6}$. Omitting the integrality restriction transforms the problem from an integer optimization model into its continuous relaxation. Although the resulting derivation may remain syntactically coherent and even yield a lower objective value, the solution can contain fractional food items and therefore may be infeasible for the original problem. The following contrast between low-temperature sampling and correct model illustrates this failure mode.

\begin{nicebox}{lightred}{Low-temperature sampling}
\small

\begin{itemize}[leftmargin=1.3em,itemsep=0.25em]
    \item \textbf{Mathematical Model:} To solve the dietitian's optimization problem, we will establish a linear programming model aimed at minimizing the cost of food while meeting the nutritional requirements of the bodybuilder.
    
    \item \textbf{Decision Variables:} ``$x_1, x_2, x_3, x_4, x_5, x_6$: Represent the quantities of Steak, Tofu, Chicken, Broccoli, Rice, and Spinach, respectively.''
    
    \item \textbf{Non-negativity Constraints:} ``$x_1, x_2, x_3, x_4, x_5, x_6 \geq 0$, as the quantities of food cannot be negative.'' \badhl{[integer restriction missing]}
    
    \item \textbf{\ldots}
\end{itemize}

\end{nicebox}

\begin{nicebox}{lightgreen}{Correct model}
\small

\begin{itemize}[leftmargin=1.3em,itemsep=0.25em]
    \item \textbf{Mathematical Model:} To solve this diet optimization problem, we will establish a linear programming model. The objective is to minimize the total cost while meeting the nutritional requirements.
    
    \item \textbf{Decision Variables:} ``$x_1, x_2, x_3, x_4, x_5, x_6$: Represent the quantities of Steak, Tofu, Chicken, Broccoli, Rice, and Spinach, respectively.''
    
    \item \textbf{Non-negativity Constraints:} ``$x_1, x_2, x_3, x_4, x_5, x_6 \geq 0$ \goodhl{and should be integers} to ensure feasible integer solutions.''
    
    \item \textbf{\ldots}
\end{itemize}

\end{nicebox}

From the above case study, a locally plausible token choice at the variable-definition stage can shift the feasible region of the downstream problem, ultimately yielding an invalid optimization model. This suggests that relying solely on immediate likelihood, even when high, can be risky. One might argue that this risk is acceptable if the LLM can robustly recover from local mistakes, i.e., occasional errors that still lead to correct final code and results. In that scenario, low-temperature sampling would remain a practical choice. However, in OR model generation, we observe a different dynamic. Once an erroneous modeling decision is made, subsequent generation steps are conditioned on a flawed mathematical structure, and backtracking seldom occurs. This reveals a deeper issue, which we examine next as our second observation.

Second, OR reasoning is highly path-dependent and admits little room for retrospective recovery. Once the model commits to an erroneous modeling choice, later steps are generated conditional on the resulting mathematical semantics. For example, a wrong variable domain, a missing constraint, or a mismatched objective term is not merely a local textual error, it also changes the optimization problem that subsequent constraints, solver code, and numerical answers are built upon.

To characterize this feature, we compare two notions of success: Pass@$k$ and CoT-Pass@$k$ \citep[see,][]{wen2025reinforcement}. For a given instance, let $s^{(1)},...,s^{(k)}$ denote $k$ independently sampled solution trajectories from the model, where each trajectory contains both an intermediate reasoning process and a final answer. Let $\mathbbm{1}_{\mathrm{ans}}(s)\in\{0,1\}$ indicate whether the final answer produced by trajectory $s$ is correct, and let $\mathbbm{1}_{\mathrm{cot}}(s)\in\{0,1\}$ indicate whether its reasoning process is correct. We define
\begin{align*}
	\mathrm{Pass@}k = \Pr\left(\exists\, i\in\{1,...,k\}, \mathbbm{1}_{\mathrm{ans}}\left(s^{(i)}\right)=1\right),
\end{align*}
and
\begin{align*}
	\mathrm{CoT\mbox{-}Pass@}k = \Pr\left(\exists\, i\in\{1,...,k\}, \mathbbm{1}_{\mathrm{ans}}\left(s^{(i)}\right)\mathbbm{1}_{\mathrm{cot}}\left(s^{(i)}\right)=1\right).
\end{align*}
Thus, Pass@$k$ measures whether at least one sampled trajectory yields the correct final answer, whereas CoT-Pass@$k$ requires at least one sampled trajectory to be correct both in its reasoning and in its final answer. 

By construction, $\mathrm{CoT\mbox{-}Pass@}k \leq \mathrm{Pass@}k$. The gap between them measures the probability that a model obtains the correct final answer without following a correct reasoning trajectory. This gap is informative because it reflects the extent to which errors in intermediate reasoning can be benign or recoverable. In answer-only reasoning tasks (e.g., math and multi choices), such a gap can be non-negligible because a model may arrive at the correct final answer through a lucky shortcut, an implicit correction, or a partially flawed chain of thought that nevertheless produces the right terminal response. In such settings, increasing $k$ can improve the chance of sampling a correct final answer even when some successful trajectories contain imperfect reasoning.

OR tasks are different. The final answer is usually the consequence of an explicit modeling-and-solving pipeline. If the variable definition, feasible region, objective function, or implementation is wrong, the resulting numerical solution is normally wrong or infeasible as well. Therefore, answer correctness is tightly coupled with reasoning correctness. Equivalently, trajectories with $\mathbbm{1}_{\mathrm{ans}}(s)=1$ but $\mathbbm{1}_{\mathrm{cot}}(s)=0$ should be rare in OR generation. In this case, Pass@$k$ and CoT-Pass@$k$ should be close to each other, because the model cannot easily compensate for a wrong intermediate formulation at the final answer stage.

\begin{figure}[t!]
\centering
\subfloat[Pass@$k$ and COT-Pass@$k$ in MAMOEasyLP.\label{fig:passk_cotpassk_mamo}]{
    \includegraphics[width=0.48\linewidth]{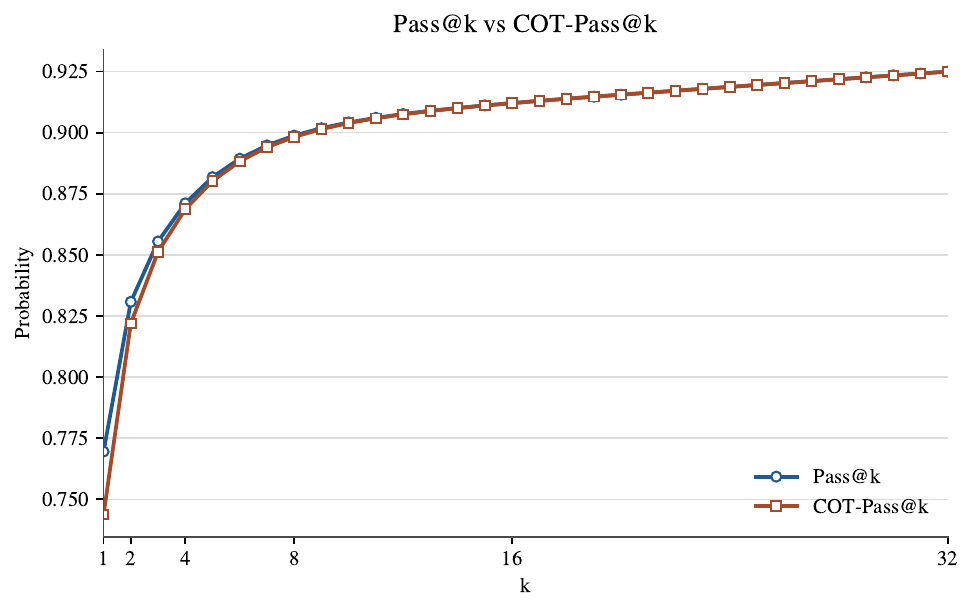}
}
\hfill
\subfloat[Pass@$k$ and COT-Pass@$k$ in AIME2025.\label{fig:passk_cotpassk_aime}]{
    \includegraphics[width=0.48\linewidth]{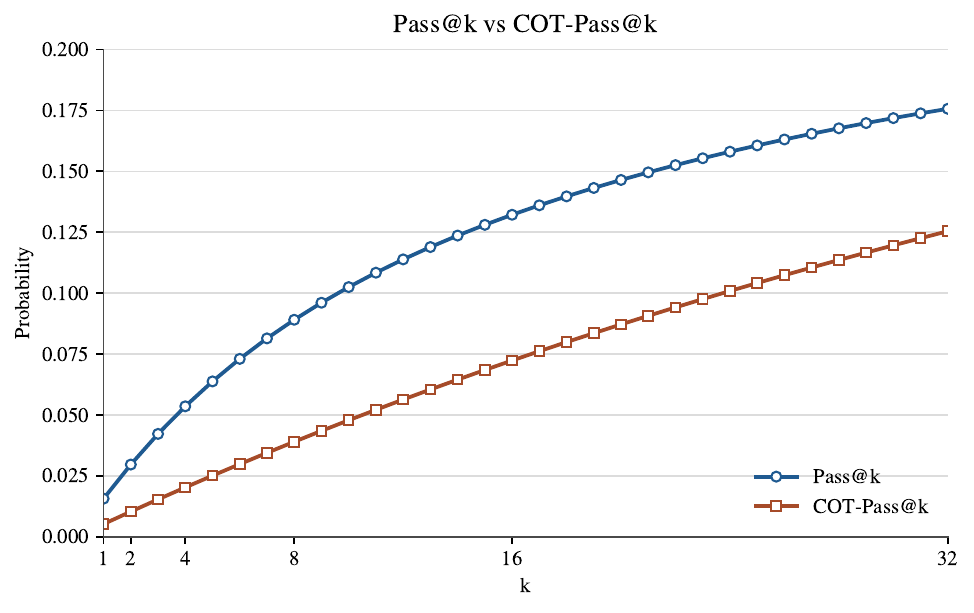}
}
\caption{Pass@$k$ versus COT-Pass@$k$ can behave differently across domains. The key distinction is whether correct final answers can be produced without a correspondingly correct reasoning process.}
\label{fig:passk_cotpassk}
\end{figure}

Figure \ref{fig:passk_cotpassk} supports this distinction. On MAMOEasyLP, Pass@$k$ and CoT-Pass@$k$ nearly coincide across different values of $k$, suggesting that correct final answers almost always arise from correct reasoning trajectories. By contrast, on AIME2025 (a mathematical reasoning benchmark), a visible gap remains between the two curves, indicating that final-answer correctness can sometimes be achieved even when the chain of thought is not fully correct. This comparison highlights a key structural property of OR generation: once the reasoning path deviates from a valid formulation, subsequent generation rarely recovers the trajectory into a correct final solution.

The case study and the Pass@$k$ versus CoT-Pass@$k$ comparison identify why myopic policy is inadequate for OR generation. Specifically, the case study shows that a locally plausible modeling choice can later invalidate the global formulation, while the Pass@$k$ versus CoT-Pass@$k$ comparison shows that, once such a choice is made, the induced error is rarely corrected by subsequent continuation. Thus, LLMs for OR questions face a compounded risk: local likelihood may favor a misleading intermediate commitment, and the effect of that commitment can persist through the variables, constraints, objective, and solver code. Therefore, at each step, reliable method should assess candidate continuations before committing to them, rather than relying only on their immediate probability under the base model.

\subsection{Motivation}
To improve the performance of LLMs, one promising approach is to generate diverse responses and then conduct reinforcement learning based on preferences, e.g., reinforcement learning from human feedback (RLHF) and reinforcement learning with verifiable rewards (RLVR). In paradigms like RLVR, model post-training is formulated as a reward maximization problem. Given a reward function $r(\mathbf{x})$ that evaluates generated answers, the objective is to learn a policy $q(\cdot)$ that maximizes expected reward while remaining close to the base model $p(\cdot)$:
\begin{align}\label{eq:opti_problem}
	\max_{q(\cdot)} \mathbb{E}_{\mathbf{x}\sim q}[r(\mathbf{x})] - \frac{1}{\lambda} D_{KL}[q(\mathbf{x}) || p(\mathbf{x})],
\end{align}
where $\lambda > 0$ and $D_{KL}$ denotes the Kullback-Leibler divergence of two different distributions. Under standard regularization assumptions, the optimal policy given reward function $r(\mathbf{x})$ admits the form
\begin{align}\label{eq:opti_police}
	q(\mathbf{x}) = \frac{1}{Z} p(\mathbf{x})\exp(\lambda r(\mathbf{x})),
\end{align}
where $\lambda$ controls the strength of reward shaping and $Z = \sum_{\mathbf{x}}p(\mathbf{x})\exp(\lambda r(\mathbf{x}))$ is a partition function. 

By defining a reward function that evaluates the final generated answers, the model learns a policy that optimizes global trajectories rather than just local tokens. Specifically, it follows from \eqref{eq:opti_police} that
\begin{align}\label{eq:local_RL}
q(x_t | x_{<t}) &= \frac{q(x_{<t},x_t)}{q(x_{<t})} = \frac{\sum_{x_{>t}}q(x_{<t},x_t,x_{>t})}{\sum_{x_t'}\sum_{x_{>t}} q(x_{<t},x_t',x_{>t})} \propto \sum_{x_{>t}} p(x_{<t},x_t,x_{>t})\exp(\lambda r(x_{<t},x_t,x_{>t}))\nonumber\\
&= p(x_{<t})p(x_t|x_{<t})\sum_{x_{>t}} p(x_{>t} | x_{\leq t})\exp(\lambda r(x_{<t},x_t,x_{>t}))\nonumber\\
&\propto p(x_t|x_{<t})\E_{x_{>t}\sim p(\cdot | x_{\leq t})}[\exp(\lambda r(\mathbf{x}))]\nonumber\\
&= p(x_t|x_{<t}) \phi(x_{\leq t}),
\end{align}
where $\phi(x_{\leq t}) = \E_{x_{>t}\sim p(\cdot | x_{\leq t})}[\exp(\lambda r(\mathbf{x}))]$ denotes a conditional expectation that aggregates the expected quality of future continuations after step $t$. Therefore, under RLVR, the preference over each candidate token is determined not only by its immediate likelihood under the base model but also by the downstream potential. This demonstrates that viewing reasoning as a trajectory-dependent process with lookahead capabilities can significantly improve the performance of the base LLMs.

However, applying reinforcement learning in OR poses significant practical challenges. Such approaches typically require iterative training, substantial computational resources, and carefully designed reward signals, which are often sparse or expensive to obtain in practice. Furthermore, there is an ongoing discussion regarding whether RLVR fundamentally endows models with novel intrinsic capabilities, or if its gains primarily arise from shifting probability toward valid reasoning paths that are already available to the base model. Given this debate, alongside the prohibitive costs of post-training, a natural question arises: can we make better inference-time decisions among competing intermediate continuations without relying on costly parameter updates?

If the fundamental advantage of RLVR lies in assessing the future consequences of local choices, we can bypass the training phase by directly incorporating this lookahead mechanism into the inference process. Instead of implicitly learning a global policy through parameter updates, we can explicitly evaluate each candidate continuation based on the quality of the complete solution that follows from it. To formalize this intuition, we introduce a value function $V$ that represents the probability of eventually arriving at a correct final answer given the current generated prefix and a candidate continuation. The next section develops this inference-time alternative and details how we approximate this ideal value function efficiently during generation.

\section{Method}\label{sec:method}

Before detailing our method, we first clarify where our framework intervenes in the generation process. When deploying LLMs, a critical distinction is whether the probability engine and the control program that assembles the final output are separable. In closed-source commercial application programming interfaces (APIs), these two components are coupled into a black box: the user submits a prompt and receives a complete sequence, with no access to the intermediate probability distributions or the token-by-token selection logic. In an open-source environment, the model exposes two primitive operations: it returns the conditional distribution $p(\cdot | prefix)$ and can draw samples from that distribution, but it does not dictate which sample is selected. The actual generation loop that chooses a specific token at each step and iteratively constructs the full sequence is an external control program that the user can modify (see Figure \ref{fig:framework_pseudocode}). As shown in the figure, this intervention relies on a resampling mechanism, which we will formally define later in this section.

This separation means we do not need to retrain the model or modify its parameters. Instead, we rewrite the external control program to embed a simulation-based lookahead search. Before committing to a candidate continuation, our algorithm uses the model's probability engine to project several steps forward, evaluating the structural stability of the resulting trajectory. It then selects the candidate whose downstream behavior is most concentrated and coherent, rather than the one with the highest immediate probability. By replacing the default myopic selection rule rather than the model weights, we steer generation toward formulations that are more likely to remain mathematically consistent.

\begin{figure*}[t]
\centering
\resizebox{\textwidth}{!}{%
\begin{tikzpicture}[
    font=\small,
    >=Stealth,
    codebox/.style={
        draw=black!60,
        thick,
        rounded corners,
        fill=white,
        drop shadow={opacity=0.12},
        minimum width=7.15cm,
        minimum height=8.3cm,
        align=left
    },
    llmbox/.style={
        draw=orange!80!black,
        thick,
        rounded corners,
        fill=orange!8,
        drop shadow={opacity=0.12},
        minimum width=9.6cm,
        minimum height=3.25cm,
        align=left
    },
    api/.style={
        draw=orange!80!black,
        rounded corners,
        fill=white,
        inner xsep=5pt,
        inner ysep=3pt,
        font=\footnotesize\ttfamily
    },
    arrow/.style={
        ->,
        thick,
        dashed,
        color=blue!65!black
    },
    simarrow/.style={
        ->,
        thick,
        dashed,
        color=orange!85!black
    },
    codeline/.style={
        anchor=west,
        font=\footnotesize\ttfamily
    },
    title/.style={
        anchor=west,
        font=\bfseries\large
    },
    subtitle/.style={
        anchor=west,
        font=\small,
        text=gray
    },
    note/.style={
        anchor=west,
        font=\footnotesize,
        text=black!70,
        text width=5.7cm,
        align=left
    }
]

\node[llmbox] (llm) at (0,0) {};

\node[font=\bfseries\large, anchor=west] at ($(llm.north west)+(3.5,-0.45)$)
    {Base LLM};

\node[api, anchor=west] at ($(llm.north west)+(0.45,-1.15)$)
    {get\_prob(x)};
\node[note] at ($(llm.north west)+(3.35,-1.15)$)
    {return $p=\Pr(\cdot\mid x)$};

\node[api, anchor=west] at ($(llm.north west)+(0.45,-1.80)$)
    {sample(p)};
\node[note] at ($(llm.north west)+(3.35,-1.80)$)
    {draw one token $z\sim p$};

\node[api, anchor=west] at ($(llm.north west)+(0.45,-2.45)$)
    {sample\_N\_tokens(p,N)};
\node[note] at ($(llm.north west)+(4.35,-2.45)$)
    {draw $N$ candidate tokens from $p$};

\node[codebox] (stdbox) at (-4.3,-6.45) {};
\node[codebox] (ourbox) at (4.3,-6.45) {};

\node[title] at ($(stdbox.north west)+(1.35,-0.45)$)
    {Standard Generation};

\node[codeline] at ($(stdbox.north west)+(0.35,-1.45)$)
    {prefix = prompt};
\node[codeline] at ($(stdbox.north west)+(0.35,-1.90)$)
    {\textbf{while} generation:};

\node[codeline] (std-getprob) at ($(stdbox.north west)+(0.65,-2.45)$)
    {p = \textcolor{blue!70!black}{LLM.get\_prob(prefix)}};

\node[codeline, text=gray] at ($(stdbox.north west)+(0.65,-3.05)$)
    {\# Local one-step decision};
\node[codeline] (std-sample) at ($(stdbox.north west)+(0.65,-3.50)$)
    {token = \textcolor{orange!85!black}{LLM.sample(p)}};
\node[codeline] at ($(stdbox.north west)+(0.65,-3.95)$)
    {prefix = prefix + token};

\node[title] at ($(ourbox.north west)+(1.75,-0.45)$)
    {Our Intervention};

\node[codeline] at ($(ourbox.north west)+(0.35,-1.35)$)
    {prefix = prompt};
\node[codeline] at ($(ourbox.north west)+(0.35,-1.80)$)
    {\textbf{while} generation:};

\node[codeline] (our-getprob-main) at ($(ourbox.north west)+(0.65,-2.30)$)
    {p = \textcolor{blue!70!black}{LLM.get\_prob(prefix)}};
\node[codeline] (our-sampleN) at ($(ourbox.north west)+(0.65,-2.75)$)
    {candidates = \textcolor{orange!85!black}{LLM.sample\_N\_tokens(p,N)}};

\node[codeline, text=gray] at ($(ourbox.north west)+(0.65,-3.35)$)
    {\# Estimate reward};
\node[codeline] at ($(ourbox.north west)+(0.65,-3.80)$)
    {\textbf{for} c \textbf{in} candidates:};
\node[codeline] at ($(ourbox.north west)+(0.95,-4.25)$)
    {sim\_prefix = prefix + c};
\node[codeline] at ($(ourbox.north west)+(0.95,-4.70)$)
    {\textbf{for} h = 1,\ldots,H:};

\node[codeline] (our-getprob-sim) at ($(ourbox.north west)+(1.25,-5.15)$)
    {q = \textcolor{blue!70!black}{LLM.get\_prob(sim\_prefix)}};
\node[codeline] (our-sample-sim) at ($(ourbox.north west)+(1.25,-5.60)$)
    {z = \textcolor{orange!85!black}{LLM.sample(q)}};
\node[codeline] at ($(ourbox.north west)+(1.25,-6.05)$)
    {sim\_prefix = sim\_prefix + z};

\node[codeline] at ($(ourbox.north west)+(0.95,-6.55)$)
    {R[c] = reward(sim\_prefix)};
\node[codeline] at ($(ourbox.north west)+(0.95,-7.00)$)
    {w[c] $\propto$ p(c) $\times$ R[c]};

\node[codeline] at ($(ourbox.north west)+(0.65,-7.55)$)
    {best\_c = resample(candidates, w)};
\node[codeline] at ($(ourbox.north west)+(0.65,-8.00)$)
    {prefix = prefix + best\_c};


\draw[arrow]
    ($(llm.west)+(0,-0.25)$)
    -- ++(-4.05,0)
    |- (std-getprob.west);

\draw[simarrow]
    ($(llm.west)+(0,-0.75)$)
    -- ++(-3.25,0)
    |- (std-sample.west);

\draw[arrow]
    ($(llm.east)+(0,-0.20)$)
    -- ++(3.25,0)
    |- (our-getprob-main.east);

\draw[simarrow]
    ($(llm.east)+(0,-0.55)$)
    -- ++(3.85,0)
    |- (our-sampleN.east);

\draw[arrow]
    ($(llm.east)+(0,-0.90)$)
    -- ++(4.45,0)
    |- (our-getprob-sim.east);

\draw[simarrow]
    ($(llm.east)+(0,-1.25)$)
    -- ++(5.05,0)
    |- (our-sample-sim.east);

\end{tikzpicture}%
}
\caption{The base LLM (top) serves strictly as a frozen probability engine. While standard methods (left) blindly follow the default myopic sampling code, our framework (right) intervenes in the decoding script by generating multiple candidates and using short-lookahead simulations to estimate structural stability. This allows us to reweight the output distribution and steer the generation toward mathematically coherent OR formulations.}
\label{fig:framework_pseudocode}
\end{figure*}
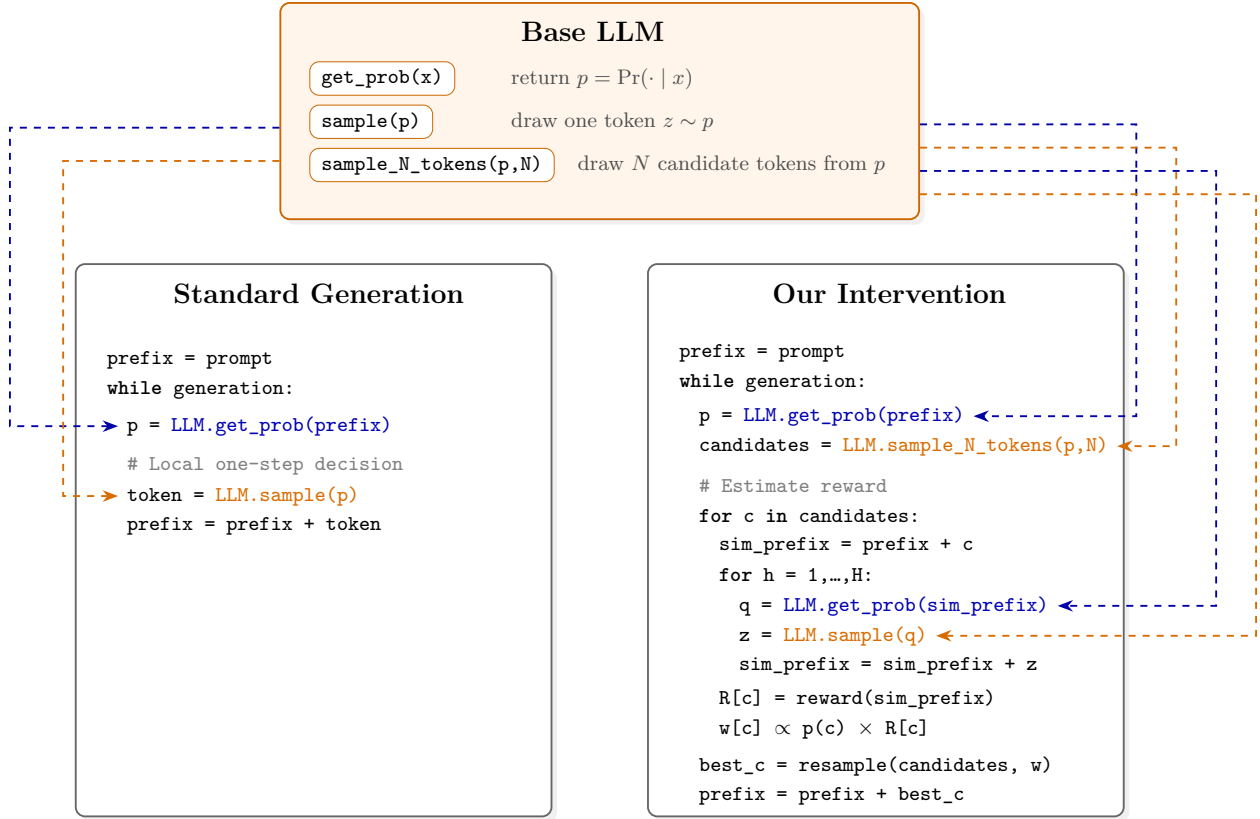

\subsection{Reward Design}\label{sec:reward_design}
As motivated in the previous section, stabilizing OR generation requires evaluating the global consequences of local choices. Suppose that at step $t$, the model has already generated a prefix $x_{<t}$, and let $\mathcal{A}_t$ denote the set of candidate tokens or blocks under consideration. Ideally, the model should choose the continuation that maximizes the probability of eventually arriving at a correct final answer. Formally, for any candidate $a \in \mathcal{A}_t$, define
\begin{align*}
    V(x_{<t} \oplus a) = \Pr\left(\text{final answer is correct} \mid x_{<t} \oplus a\right),
\end{align*}
where $x_{<t} \oplus a$ denotes the new prefix formed by appending $a$ to $x_{<t}$. The ideal decision rule is then $a_t^\star \in \arg\max_{a \in \mathcal{A}_t} V(x_{<t} \oplus a)$, and the next output is chosen as $a_t^\star$.

In principle, one can estimate $V(x_{<t} \oplus a)$ by simulation. That is, for each candidate $a$, we temporarily fix $a$ as the continuation at step $t$, sample a sufficiently large number of full trajectories from $x_{<t} \oplus a$, and use the empirical success rate of the final answers as an estimator of $V(x_{<t} \oplus a)$. This corresponds to propagating terminal correctness information back to the current local decision and related ideas have been explored in \cite{zhao2024probabilistic}. However, $V(x_{<t} \oplus a)$ is generally difficult to use during generation. In OR tasks, correctness is typically a terminal property, i.e., whether a candidate is truly valid can usually be determined only after the complete formulation, implementation, and final answer have been generated. As a result, directly estimating $V$ for each intermediate candidate would require repeated full-path evaluation, which is not only computationally expensive but also poorly suited to timely decision-making during inference.

This observation motivates the need for a surrogate reward of $V$. Ideally, such a proxy should satisfy two requirements. First, it should be informative about the true terminal value $V(x_{<t} \oplus a)$, so that improving the proxy also tends to improve final correctness. Second, it should be estimable from partial future information, without requiring the model to complete and verify the entire remaining path for every candidate. In other words, we seek a signal that is correlated with eventual success, yet can already be extracted from partial future information.

Our intuition is that downstream uncertainty provides such a signal, which has also been studied by \cite{farquhar2024detecting}. Importantly, the relevant uncertainty in OR tasks is not one-dimensional. It may appear in several forms, including divergent variable semantics, inconsistent objective formulations, incompatible constraints, conflicting implementation choices, or other structural mismatches in the downstream reasoning process. A candidate that appears locally plausible may still induce substantial branching across these dimensions, indicating that it is not on a stable formulation path. By contrast, a strong candidate is more likely to support a concentrated and coherent set of future continuations. Moreover, these uncertainties often reveal themselves over a relatively short horizon. In OR generation, many structural errors do not remain latent until the final answer. For example, an incorrect variable definition may quickly lead to different objective terms, mismatched constraints, or inconsistent code structure in the subsequent steps. Therefore, it is often unnecessary to simulate the entire remaining path in order to assess candidate quality.

To illustrate this relationship between uncertainty and correctness, we consider a concrete example in Figure~\ref{fig:uncertainty_correctness}. At ``variable definition'' step in an MAMOEasyLP instance, we examine 16 candidate continuations and, for each candidate, simulate 32 full downstream rollouts. We then manually cluster these rollouts and compute the empirical entropy $\widehat{H}(a) = - \sum_{c=1}^{C(a)} \widehat{\pi}_c(a)\log \widehat{\pi}_c(a)$, where $\widehat{\pi}_c(a)$ is the fraction of rollouts assigned to structural cluster $c$, and $C(a)$ is the number of induced clusters. We also estimate the empirical success rate $\widehat{V}(a) = \frac{1}{32}\sum_{m=1}^{32}\mathbbm{1}_{\mathrm{ans}}\bigl(\tau_a^{(m)}\bigr)$, where $\tau_a^{(m)}$ is the $m$-th rollout under candidate $a$. As shown in the figure, lower structural entropy tends to be associated with higher success rate. For example, the candidate with $\widehat{H}=0.68$ achieves a success rate of 87.5\%, whereas the candidate with $\widehat{H}=1.85$ achieves only 9.4\%. We can see that candidates $a_{10}$ and $a_{14}$ appear highly similar at the local level, yet their downstream uncertainty and eventual correctness differ substantially. This does not prove that uncertainty alone determines correctness, but it strongly suggests that downstream uncertainty contains useful information about whether a candidate continuation remains on a valid OR reasoning path.

\begin{figure}[t!]
\centering
\includegraphics[width=\linewidth]{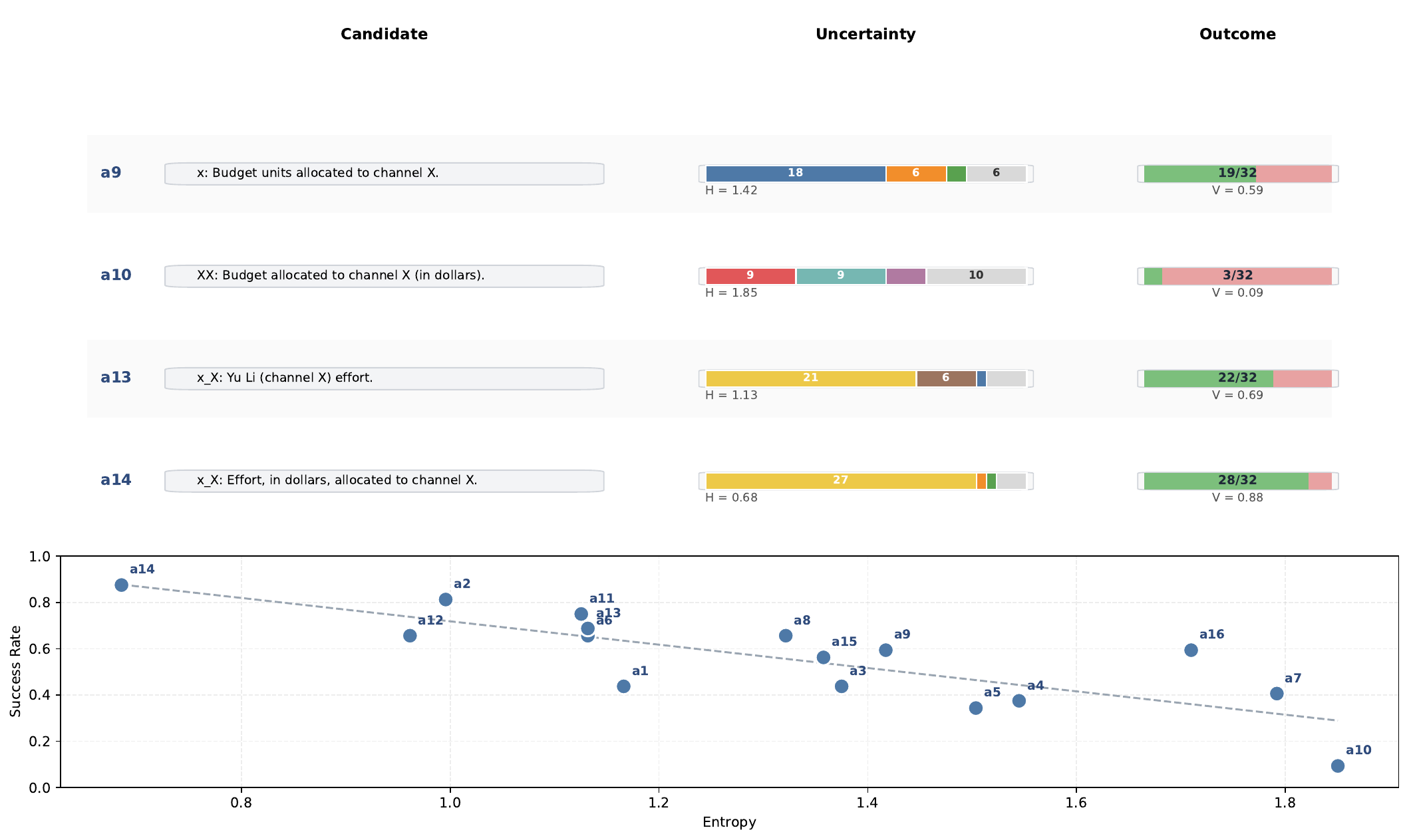}
\caption{A motivating illustration of the relationship between downstream dispersion and empirical correctness. At variable-definition step, we compare 16 candidate continuations and simulate 32 full downstream rollouts for each candidate.}
\label{fig:uncertainty_correctness}
\end{figure}

It is not sufficient to only use the uncertainty because a candidate may have low downstream entropy simply because the model is confidently following a wrong but internally consistent path. Hence, we also need to retain the local preference of the base model, i.e., whether the current token or block appears natural under $p(\cdot)$. This leads to reward functions that combine \emph{local plausibility} and \emph{future stability}. For simplicity, we present the method at token-level notation in the following. That is, at step $t$ we consider selecting the next token $x_t$. It is worth noting that in practical OR generation, it is often more meaningful to apply the same idea at the block level, since a single token rarely reveals whether a formulation choice is sound, whereas a complete variable declaration, a full constraint, or an objective term often does. All definitions below can be extended immediately if one replaces $x_t$ by a semantically meaningful block.

Let $H_t = - \mathbb{E}_{x_{>t}\sim p(\cdot | x_{\leq t})}\left[\log p(x_{>t} | x_{\leq t})\right]$ denote the conditional entropy of future continuations given the current prefix $x_{\le t}$. A large value of $H_t$ indicates that the current choice leads to many competing future trajectories, while a small value indicates that the downstream continuation is more concentrated and therefore more stable. Next, we define the entropy reward which is similar to $\phi$ in \eqref{eq:local_RL}
\begin{align}\label{eq:entropy_reward}
	\phi_{\mathrm{entropy},\alpha,\beta}(x_{\leq t}) = \exp\left(\alpha \log p(x_t | x_{<t}) - \beta H_t\right),
\end{align}
where $\alpha \geq 0$ and $\beta \geq 0$ are tuning parameters. The first term on the right hand side (RHS) of \eqref{eq:entropy_reward} encourages locally plausible continuations and reduces the chance of selecting unnatural tokens. The second term on the RHS of \eqref{eq:entropy_reward} penalizes prefixes whose future paths are highly dispersed. Thus, $\phi_{\mathrm{entropy},\alpha,\beta}$ favors continuations that are both natural under the base model and structurally stable under future rollout.
This reward directly matches the intuition developed above. In OR tasks, the ideal continuation is not merely one that looks locally likely, but one that keeps the future formulation path concentrated around coherent variable definitions, compatible constraints, or a stable objective structure. The entropy term is designed to capture exactly this notion of downstream stability.

Besides explicit entropy penalization, it is also useful to consider a more likelihood-oriented reward that favors prefixes whose future probability mass is concentrated on a small number of strong continuations. This leads to the power reward
\begin{align}\label{eq:power_reward}
	\phi_{\mathrm{power},\alpha}(x_{\leq t}) = \exp\left(\alpha \log p(x_t | x_{<t}) + \log \mathbb{E}_{x_{>t}\sim p(\cdot | x_{\leq t})} \left[p(x_{>t} | x_{\leq t})^{\alpha}\right]\right),
\end{align}
where $\alpha \geq 0$. The interpretation of $\phi_{\mathrm{power},\alpha}$ is slightly different from that of the entropy reward. Instead of explicitly subtracting uncertainty, it rewards prefixes whose lookahead distribution assigns substantial mass to a few high-probability future trajectories. In this sense, it is a softer and more likelihood-driven proxy for continuation quality. When the downstream probability mass is diffuse, the power moment is relatively small, while when the downstream mass is concentrated on a few strong continuations, the power moment becomes large. Hence, $\phi_{\mathrm{power},\alpha}$ also prefers stable prefixes, but it does so through high-probability concentration rather than explicit entropy regularization.

The relationship between the two rewards can be clarified by Jensen's inequality. Since the logarithm is concave, we have
\begin{align*}
	\log \mathbb{E}_{x_{>t}\sim p(\cdot | x_{\leq t})} \left[p(x_{>t} | x_{\leq t})^{\alpha}\right] \geq \mathbb{E}_{x_{>t}\sim p(\cdot | x_{\leq t})} \left[\log p(x_{>t} | x_{\leq t})^{\alpha}\right] = -\alpha H_t.
\end{align*}
Therefore, the power reward can be viewed as a more aggressive high-mass surrogate, whereas the entropy reward explicitly penalizes dispersion. The former emphasizes strong future likelihood concentration, while the latter more directly targets future uncertainty.

Importantly, our goal here is not to claim that one of these two rewards is universally superior. Rather, they represent two principled ways of turning rollout-based downstream information into an inference-time proxy for trajectory quality. The entropy reward is closer to the structural-stability intuition, while the power reward yields a cleaner likelihood-based transformation that will be especially convenient in the next section. Both rewards above are still defined through expectations over future continuations, and exact evaluation remains intractable. The next section therefore develops an efficient approximation scheme based on limited lookahead and importance resampling, together with the corresponding theoretical justification.

\subsection{Importance Resampling}

In the following, we introduce how to estimate rewards and how to perform importance resampling. Note that although $\phi_{entropy, \alpha, \beta}$ and $\phi_{power, \alpha}$ in Section \ref{sec:reward_design} formally rely on the expectation over future continuations, it differs fundamentally from the objective in RLVR. Specifically, RLVR relies on terminal, binary correctness (e.g., whether the final answer is TRUE or FALSE) that can only be determined after completing the entire reasoning pipeline. In contrast, $\phi_{entropy, \alpha, \beta}$ and $\phi_{power, \alpha}$ are continuous metrics designed to capture downstream uncertainty and prefix stability. In OR and strongly path-dependent tasks, this structural uncertainty typically manifests within a relatively short continuation horizon. For example, an incorrect variable definition will quickly reveal itself through inconsistent objective terms or mismatched constraints shortly after it is generated, without needing to run the final solver (see Appendix \ref{app:examples}). Moreover, distant future tokens will also be evaluated in subsequent lookahead windows as the generation advances. Therefore, we can only consider the next $H$ tokens to balance estimation accuracy and computational efficiency.

Given a truncated reward function $\phi_H(x_{\leq t})$ which may be either the entropy reward or the power reward, we now describe how to incorporate this signal into inference. Since directly enumerating all possible continuations is infeasible, we adopt an importance resampling strategy at each generation step. Fix a step $t < T$ and an already-generated prefix $x_{<t}$. We define the local truncated selection rule
\begin{align}\label{eq:truncated_reward}
q_H(a | x_{<t}) = \frac{p(a | x_{<t})\phi_H(x_{<t}\oplus a)}{\sum_{a' \in \mathcal{A}_t} p(a' | x_{<t})\phi_H(x_{<t}\oplus a')}, \quad a \in \mathcal{A}_t, \quad H < T - t.
\end{align}
This distribution assigns larger probability mass to candidates with higher reward values, so candidates that are both locally plausible under the base model and more promising under the chosen lookahead-based reward are more likely to be selected. In particular, when $\phi_H$ is chosen as the power reward, the resulting resampled distribution can be viewed as an approximation to a power version of the original distribution, which is similar to the idea in \cite{karan2025reasoning}. Equation \eqref{eq:truncated_reward} is local in the sense that it reweights only the current-step decision, rather than defining a separate global target distribution over complete trajectories. 

To apply the selection rule at step $t$, we draw $N$ candidate tokens independently with replacement from $p(\cdot | x_{<t})$, denoted by $x_t^1, ..., x_t^N$, and write $x_{\leq t}^i = (x_{<t}, x_t^i)$ for the resulting candidate prefixes. For each $x_{\leq t}^i$, we generate $M$ rollouts to estimate $\phi_H(x_{\leq t}^i)$, denoted by $\phi_{M,H}(x_{\leq t}^i)$. The sampling procedure is shown in Figure \ref{fig:sampling}. We then assign each candidate a nonnegative score $\widetilde{w}_i = \phi_{M,H}(x_{\leq t}^i)$ and normalize these scores to obtain $w_i = \widetilde{w}_i/\sum_{k=1}^N \widetilde{w}_k$. Finally, we sample an index $I$ according to the categorical distribution $\{w_i\}_{i=1}^N$ and set $x_{t,N,M} = x_t^I$. This procedure preserves the support of the base model while shifting probability mass toward candidates that receive larger reward values, thereby favoring continuations that are more promising under the chosen lookahead criterion. In the following, we provide Theorem \ref{thm:prefix_ir_strong_full} to show that as the number of sampled candidates and rollouts increases, the resampled token converges in distribution to the target rule $q_H(\cdot|x_{<t})$. Thus, the practical procedure can be viewed as a Monte Carlo approximation to the ideal local reweighted selector. The proof of Theorem \ref{thm:prefix_ir_strong_full} is provided in Appendix \ref{app:proof}.

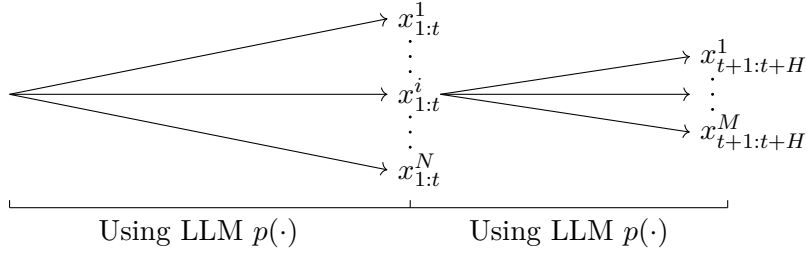
\begin{figure}[t!]
\begin{center}
\begin{tikzpicture}

\draw [->] (0,0) -- (5,1);
\draw [->] (0,0) -- (5,0);
\draw [->] (0,0) -- (5,-1);

\node [right] at (5,1) {$x_{1:t}^1$};
\node [right] at (5,0) {$x_{1:t}^i$};
\node [right] at (5,-1) {$x_{1:t}^N$};

\draw[fill] (5.3,0.7) circle [radius=0.012];
\draw[fill] (5.3,0.5) circle [radius=0.012];
\draw[fill] (5.3,0.3) circle [radius=0.012];

\draw[fill] (5.3,-0.7) circle [radius=0.012];
\draw[fill] (5.3,-0.5) circle [radius=0.012];
\draw[fill] (5.3,-0.3) circle [radius=0.012];

\draw [->] (5.7,0) -- (9,0.5);
\draw [->] (5.7,0) -- (9,0);
\draw [->] (5.7,0) -- (9,-0.5);

\node [right] at (9,0.5) {$x_{t+1:t+H}^{1}$};
\node [right] at (9,-0.5) {$x_{t+1:t+H}^{M}$};

\draw[fill] (9.3,0.2) circle [radius=0.012];
\draw[fill] (9.3,0) circle [radius=0.012];
\draw[fill] (9.3,-0.2) circle [radius=0.012];

\draw [-] (0,-1.5) -- (9.5,-1.5);
\draw [-] (0,-1.5) -- (0,-1.4);
\draw [-] (5.3,-1.5) -- (5.3,-1.4);
\draw [-] (9.5,-1.5) -- (9.5,-1.4);

\node [below] at (2.5,-1.5) {Using LLM $p(\cdot)$};
\node [below] at (7.4,-1.5) {Using LLM $p(\cdot)$};

\end{tikzpicture}
\end{center}
\caption{Illustration of the sampling procedure at time step $t<T$.
The prefix $x_{1:t-1}$ is fixed, and multiple candidate tokens at position $t$
are evaluated using lookahead rollouts.}
\label{fig:sampling}
\end{figure}

\begin{theorem}\label{thm:prefix_ir_strong_full}
Fix a step $t$, a prefix $x_{<t}$, and a lookahead horizon $H<T-t$. Assume that $\mathcal{A}_t$ is finite and that $p(a | x_{<t}) > 0$ for all $a\in\mathcal A_t$. Then, as $N,M \to \infty$, we have $x_{t,N,M} \Rightarrow q_H(\cdot | x_{<t})$, where $\Rightarrow$ means convergence in distribution.
\end{theorem}

The theorem provides the basic justification for our inference procedure. Although exact evaluation of the local reweighted rule is intractable, the proposed importance resampling scheme asymptotically recovers the same decision distribution. In this sense, the algorithm should be understood not as an ad hoc generation heuristic, but as a statistically consistent approximation to the ideal reward-guided local selector.

Note that Theorem \ref{thm:prefix_ir_strong_full} is stated at the token level. However, in practice, the quality of a single candidate token is often difficult to assess, since many meaningful modeling decisions are only expressed after several consecutive tokens. We therefore implement the procedure at the block level. Specifically, at each decision point, we sample $N$ candidate blocks of length $B$ from the base model, and for each candidate block we further simulate $M$ rollout continuations of length $H$ in order to construct a Monte Carlo estimate of the truncated reward. This blockwise implementation is more consistent with the way optimization formulations are actually generated, where semantically meaningful units often appear as variable declarations, constraint clauses, or code fragments rather than isolated tokens.

Another advantage of the blockwise implementation is its lower computational overhead. Suppose the final generated sequence has length $T$, and for simplicity assume that $B$ divides $T$. Then the number of blockwise resampling stages is $T/B$. At each stage, generating $N$ candidate blocks costs $NB$ tokens, while the lookahead simulation costs $NMH$ tokens. Hence, the total token consumption is
\begin{align*}
	\E[{\mbox{token}}] = \frac{T}{B}(NB + NMH) = NT + \frac{NMH}{B}T.
\end{align*}
When $B = 1$, $\E[{\mbox{tokens}}] = N(MH+1)T$, while when $B = H$, this reduces to $NT + NMT = (1+M)NT$. Therefore, block-level lookahead provides a more meaningful evaluation unit while keeping the additional inference cost manageable. The procedure is presented in Algorithm \ref{alg:blockwise_ir}.

\begin{algorithm2e}[t]
\caption{Importance Resampling Generation}
\label{alg:blockwise_ir}
\DontPrintSemicolon
\KwIn{Base model $p(\cdot)$, initial prompt $x_{<1}$, block length $B$, rollout horizon $H$, numbers of candidates and rollouts $(N,M)$, total generation length $T$, truncated reward estimator $\phi_{M,H}(\cdot)$}
\KwOut{Generated sequence $x_{1:T}$}
$t \gets 1$\;
\While{$t \leq T$}{
    \For{$i=1,...,N$}{
        Sample a candidate block $b_i \sim p(\cdot | x_{<t})$ of length $B$\;
        \For{$m=1,\dots,M$}{
            Sample a rollout $r_{i,m} \sim p(\cdot | x_{<t} \oplus b_i)$ of length $H$\;
        }
        Compute $\phi_{M,H}(x_{<t} \oplus b_i)$ using $\{r_{i,m}\}_{m=1}^M$\;
        $\tilde{w}_i \gets \phi_{M,H}(x_{<t} \oplus b_i)$\;
    }
    $w_i \gets \tilde{w}_i / \sum_{k=1}^N \tilde{w}_k$, for $i=1,...,N$\;
    Sample $I \sim \mathrm{Categorical}(w_1,...,w_N)$\;
    $x_{<t+B} \gets x_{<t} \oplus b_I$\;
    $t \gets t+B$\;
    \If{$b_I$ contains the EOS token}{
        \textbf{break}\;
    }
}
\textbf{return} $x_{<t}$\;
\end{algorithm2e}

\section{Numerical Experiments}\label{sec:numerical}
We evaluate inference-time sampling strategies for improving reasoning performance of a fixed base LLM. Following standard practice in prior work, we report pass@1 accuracy as the primary evaluation metric. We conduct experiments on multiple operations research benchmarks:

\begin{itemize}
    \item \textbf{NL4OPT}: a natural-language-to-optimization benchmark focusing on linear programming formulation.

   \item \textbf{MAMO}: a benchmark for mathematical optimization modeling, including both EasyLP and ComplexLP subsets.

   \item \textbf{IndustryOR}: an industry-oriented operations research benchmark consisting of domain-specific optimization and decision-making problems.
\end{itemize}

We evaluate our methods using \textbf{ORLM-LLaMA-3-8B}, a model explicitly tailored for optimization modeling and decision-making tasks. Evaluations are conducted in an instruction-following generation framework, and based on $(N,M,B,H) = (10,4,8,8)$ in Table \ref{tab:orlm_results} and $(N,M,B,H) = (10,4,48,48)$ in Figure \ref{fig:sensitivity_analysis}. All experiments are executed on a single 48GB NVIDIA virtual GPU (vGPU). We benchmark against the following inference-time sampling strategies:

\begin{itemize}
  \item \textbf{Standard sampling:} Autoregressive sampling from the learned model distribution $p(\cdot)$ with a temperature of $\tau = 1$.

  \item \textbf{Low-temperature sampling:} Temperature scaled to $\tau = 0.25$, serving as a robust, training-free baseline for distribution sharpening and stochasticity reduction.

  \item \textbf{Power reward sampling:} A lookahead-driven strategy guided by the reward function $\phi_{power,\alpha}$. At each generation step, we sample $N$ blocks of length $B$. For each block, we execute $M$ stochastic lookahead rollouts over a horizon $H$, resampling based on the estimated lookahead weights.

  \item \textbf{Entropy reward sampling:} A variant of the lookahead strategy incorporating the reward function $\phi_{entropy,\alpha,\beta}$. The algorithmic procedure parallels the power reward but inherently regularizes the generation trajectory via entropy minimization.
\end{itemize}

Table~\ref{tab:orlm_results} details the empirical performance of ORLM-LLaMA-3-8B across various inference configurations on multiple OR benchmarks. We evaluate two primary metrics: success rate, indicating the generation of syntactically valid and executable programs, and correctness rate, verifying whether the derived numerical solution perfectly aligns with the ground truth.
Standard autoregressive sampling ($\tau = 1.0$) yields suboptimal correctness, notably degrading on highly constrained benchmarks such as MAMOComplexLP (22.7\%) and IndustryOR (24.0\%). Unlike open-ended text generation, OR modeling presents a highly unforgiving output space where a single aberrant token can render an entire mathematical formulation structurally invalid or logically infeasible. The inherent stochasticity of standard sampling exacerbates this fragility. While dampening the temperature ($\tau = 0.25$) stabilizes token-level emission and consistently elevates both metrics across all benchmarks, it remains inherently myopic, lacking the capacity for long-horizon logical planning.

As demonstrated in Table 1, lookahead-augmented sampling strategies fundamentally overcome the limitations of myopic policy, establishing superior performance across complex OR datasets. In rigorous OR problems, the generated trajectory is highly susceptible to cascading structural failures. Lookahead mechanisms explicitly mitigate this by proactively anticipating downstream states, thereby pruning erroneous trajectories and steering the model away from structural dead-ends before irreversible token commitments are made. Consequently, power reward sampling achieves absolute correctness improvements of 17.6\% and 9.0\% over standard sampling on MAMOComplexLP and IndustryOR, respectively. Furthermore, entropy reward sampling enhances this trajectory robustness by jointly optimizing for likelihood and future uncertainty. This regularization enables the highest overall correctness on IndustryOR (37.0\%). Ultimately, these findings underscore that uncertainty-aware lookahead signals are indispensable for navigating the rigid constraints of mathematical optimization and ensuring robust autoregressive generation.

\begin{table}[htbp!]
\centering
\caption{Performance of ORLM-LLaMA-3-8B under different inference strategies.}
\label{tab:orlm_results}
\renewcommand{\arraystretch}{2}

\resizebox{\textwidth}{!}{
\begin{tabular}{l l c c c c c c c c c}
\toprule
\multirow{2}{*}{\textbf{Model}} & \multirow{2}{*}{\textbf{Method}} & \multirow{2}{*}{\textbf{Params}} &
\multicolumn{2}{c}{\textbf{NL4OPT}} & \multicolumn{2}{c}{\textbf{MAMOEasyLP}} & \multicolumn{2}{c}{\textbf{MAMOComplexLP}} & \multicolumn{2}{c}{\textbf{IndustryOR}} \\
\cmidrule(lr){4-5} \cmidrule(lr){6-7} \cmidrule(lr){8-9} \cmidrule(lr){10-11}
& & & \textbf{Succ.} & \textbf{Corr.} & \textbf{Succ.} & \textbf{Corr.} & \textbf{Succ.} & \textbf{Corr.} & \textbf{Succ.} & \textbf{Corr.} \\
\midrule
Without lookahead & Standard sampling & $T=1.0$ & 93.9\% & 70.2\% & 97.4\% & 75.3\% & 72.0\% & 22.7\% & 71.0\% & 24.0\% \\
& Low-temperature sampling & $T=0.25$ & 98.0\% & 76.3\% & 99.5\% & 79.0\% & 82.9\% & 36.5\% & 83.0\% & 31.0\% \\
\midrule
With lookahead & Power reward sampling & $\alpha=3$ & 98.4\% & 76.7\% & \textbf{100\%} & \textbf{80.2\%} & \textbf{83.9\%} & \textbf{40.3\%} & 81.0\% & 33.0\% \\
& Entropy reward sampling & $(\alpha,\beta)=(3,0.2)$ & \textbf{98.8\%} & \textbf{78.8\%} & 99.7\% & 79.1\% & 81.0\% & 37.4\% & \textbf{83.0\%} & \textbf{37.0\%} \\
\bottomrule
\end{tabular}
}
\end{table}

Figure \ref{fig:sensitivity_analysis} reports the parameter sensitivity of the two proposed reward-based sampling schemes. Panel \ref{fig:power_alpha} shows the performance of the power reward sampling method. As $\alpha$ increases from 0 to 3, the accuracy improves substantially, indicating that the power reward helps guide the model toward higher-quality trajectories. For larger values, such as $\alpha=7$ and $\alpha=99$, the performance becomes stable with only minor fluctuations. This pattern suggests that the power-reward mechanism is fairly robust once $\alpha$ is chosen beyond a small range.

Panels \ref{fig:entropy_alpha} and \ref{fig:entropy_beta} present the entropy reward sampling method under different combinations of $\alpha$ and $\beta$. As shown in Panel \ref{fig:entropy_alpha}, where $\beta$ is fixed, the effect of $\alpha$ is nontrivial. Increasing $\alpha$ from a very small value often improves performance, suggesting that stronger reward shaping can be beneficial. However, the gain becomes limited once $\alpha$ reaches a moderate range. This suggests that $\alpha$ mainly controls the strength of the reward signal, while its effectiveness depends on its interaction with the entropy coefficient $\beta$.

Panel \ref{fig:entropy_beta} examines the effect of $\beta$ for fixed values of $\alpha$. An important observation is that the case $\beta=0$ is generally not optimal, which indicates that the entropy reward provides useful guidance in the sampling process. In this sense, the entropy term is not merely an auxiliary regularizer, but an effective part of the reward design. At the same time, excessively large values of $\beta$ may lead to performance degradation. One possible reason is that a strong entropy penalty makes the algorithm overly restrictive toward variations along subsequent trajectories. In sequential reasoning or search problems, different trajectories may look different at the surface level while still being semantically consistent. Penalizing all such variations in the same way may therefore suppress valid and informative diversity, which in turn hurts performance. This observation suggests that the main limitation lies not in the use of entropy itself, but in the level at which entropy is measured. A natural direction for future research is therefore to move from surface-level entropy to a semantic-level entropy measure \citep[see, e.g.,][]{farquhar2024detecting} that can distinguish harmful instability from semantically consistent variation.

\begin{figure}[t!]
  \centering
  \subfloat[Power reward sampling\label{fig:power_alpha}]{
    \includegraphics[width=0.50\linewidth]{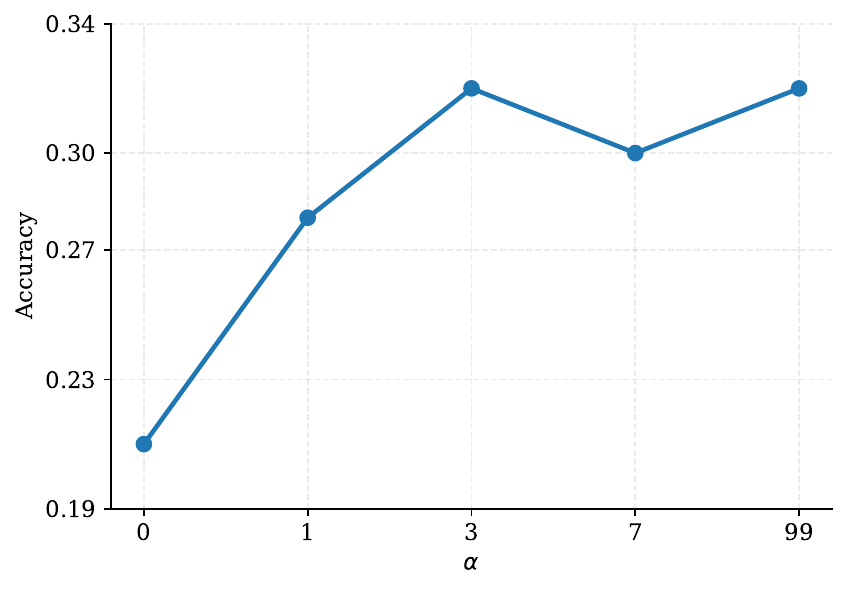}
  }
  \hfill
  \subfloat[Entropy reward sampling\label{fig:entropy_alpha}]{
    \includegraphics[width=0.48\linewidth]{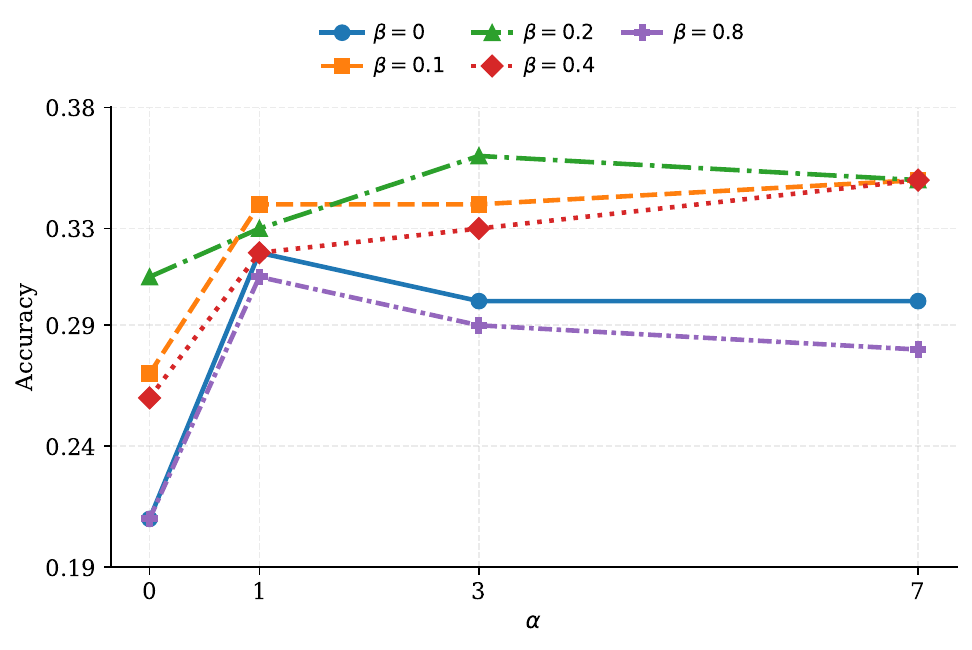}
  }
  \hfill
  \subfloat[Entropy reward sampling\label{fig:entropy_beta}]{
    \includegraphics[width=0.48\linewidth]{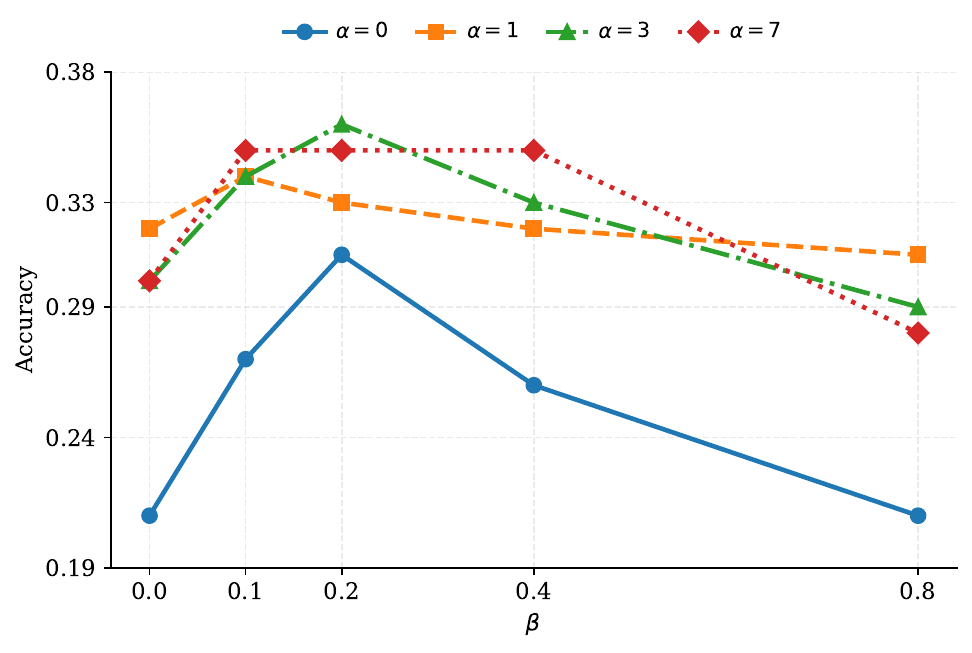}
  }
  \caption{Sensitivity analysis of reward scaling hyperparameters $\alpha$ and $\beta$ across different lookahead sampling strategies on IndustryOR.}
  \label{fig:sensitivity_analysis}
\end{figure}

\begin{figure}[htbp!]
  \centering
  \includegraphics[width=\linewidth]{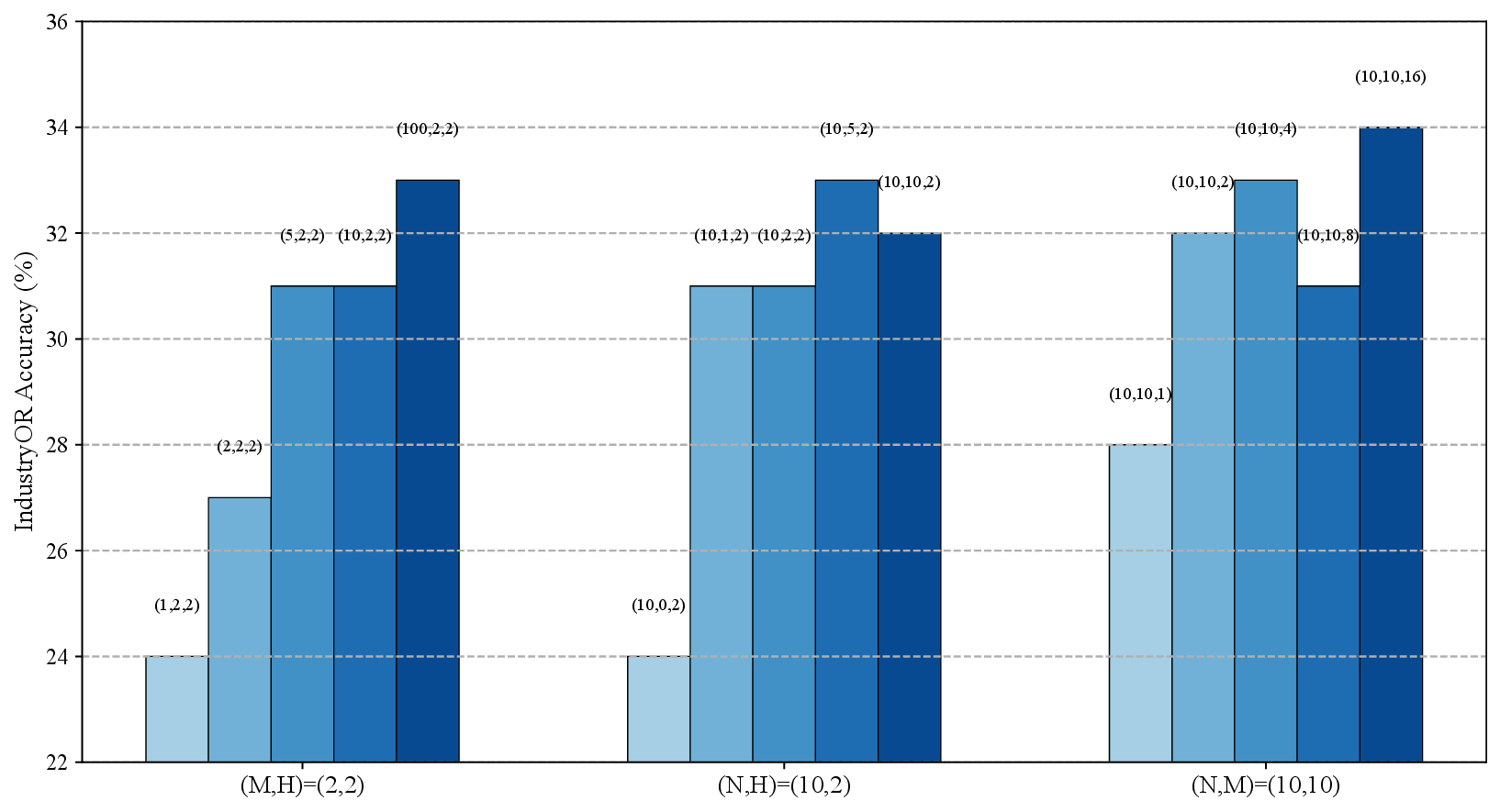}
  \caption{Pass@1 accuracy on IndustryOR using power reward sampling under various hyperparameter configurations. Each bar is labeled as $(N,M,H)$. The three groups show sensitivity to: (left) $N$ with $(M,H) = (2,2)$; (middle) $M$ with $(N,H) = (10,2)$; and (right) $H$ with $(N,M) = (10,10)$.}
  \label{fig:powerlaw_ablation}
\end{figure}

Figure \ref{fig:powerlaw_ablation} reports the Pass@1 accuracy of the power reward sampling method on the IndustryOR benchmark under different hyperparameter configurations. To assess the role of each hyperparameter separately, the three groups on the horizontal axis fix $(M,H)=(2,2)$, $(N,H)=(10,2)$, and $(N,M)=(10,10)$, respectively, while varying the remaining parameter.

The figure shows that performance improves as the budget increases from very small values to moderate values. For example, when $(M,H)=(2,2)$ is fixed, increasing $N$ from $1$ to $100$ raises the accuracy from $24\%$ to $33\%$. Similar gains are observed when varying $M$ with $(N,H)=(10,2)$ fixed, and when varying $H$ with $(N,M)=(10,10)$ fixed. This pattern indicates that the proposed method benefits from additional sampling and lookahead, especially when the initial budget is limited. At the same time, the improvement becomes noticeably smaller once $N$ and $M$ reach moderate levels. In particular, strong performance is already attained at configurations such as $(N,M,H)=(10,5,2)$, $(10,10,2)$, and $(10,10,4)$, while the best result is achieved at $(10,10,16)$. The gap between these moderate settings and the largest tested configurations is limited. This suggests that the effectiveness of the method does not depend on large values of $N$ or $M$.

This observation is important from a deployment perspective. In practical applications, increasing $N$ and $M$ directly raises computational cost and inference latency. The results show that the proposed method can achieve strong accuracy without requiring large candidate sets or extensive rollout budgets. Hence, the method offers a favorable balance between solution quality and computational burden, which makes it more suitable for real-world deployment in latency-sensitive settings.

\section{Conclusion}\label{sec:conclusion}
In this paper, we study how to improve the reliability of LLM-based formulation generation for OR tasks. The main difficulty in this setting is that formulation quality is determined by the consistency of the entire generation path, whereas conventional methods make only local decisions and therefore cannot adequately control trajectory-level risk. To address this issue, we propose a training-free inference framework that performs distribution control during generation. The key idea is to use lookahead-based rewards to evaluate the quality of partial trajectories and then apply importance resampling to shift probability mass toward more reliable continuation paths. In this way, the model is guided not only by immediate probabilities, but also by the downstream stability of the generated formulation. Across multiple benchmarks, the proposed method consistently improves formulation accuracy, while requiring only a moderate increase in inference cost.

For future work, the current framework measures uncertainty mainly through naive entropy. A natural next step is to incorporate semantic entropy \citep{farquhar2024detecting} into uncertainty estimation, so that the reward can better distinguish genuinely unstable trajectories from semantically equivalent ones. Moreover, the allocation of computational budget across $N$ and $M$ deserves further study. Since the method has a nested structure, with candidate generation and rollout evaluation interacting with each other, it may be possible to draw on ideas from the nested simulation literature \citep[e.g.,][]{Gordy2010,zhang2022bootstrap} to design principled budget allocation rules.

\bibliographystyle{chicago}
\bibliography{mybibfile}

\newpage

\appendix
\renewcommand{\appendixname}{Appendix~\Alph{section}}

\section{Proof of Theorem \ref{thm:prefix_ir_strong_full}}\label{app:proof}
\textbf{Proof:}
Denote $p_t(\cdot)=p(\cdot | x_{<t})$, $q_{t,H}(\cdot)=q_H(\cdot | x_{<t})$ and for each $M$,
\begin{align}\label{eq:densityMH}
q_{t,M,H}(a) = \frac{p_t(a)\phi_{M,H}(x_{<t}\oplus a)}{\sum_{a'\in\mathcal A_t} p_t(a')\phi_{M,H}(x_{<t}\oplus a')}, \quad a\in\mathcal A_t.
\end{align}
For any measurable set $A\subseteq\mathcal A_t$, let
\begin{align*}
R_{N,M}(A) = \frac{\sum_{i=1}^N q_{t,M,H}(x_t^i)/p_t(x_t^i)\1\{x_t^i\in A\}}{\sum_{j=1}^N q_{t,M,H}(x_t^j)/p_t(x_t^j)}.
\end{align*}
Then
\begin{align*}
\Pr(x_{t,N,M}\in A) = \E\left[\Pr\left(x_{t,N,M}\in A \mid x_t^1,\dots,x_t^N,\phi_{M,H}\right)\right] = \E\left[\sum_{i=1}^N w_i \1\{x_t^i\in A\}\right] = \E\left[R_{N,M}(A)\right],
\end{align*}
where $w_i ={\phi_{M,H}(x_{\leq t}^i)} / {\sum_{j=1}^N \phi_{M,H}(x_{\leq t}^j)}$.

Since $\mathcal A_t$ is finite and for every $a\in\mathcal A_t$, we have $\phi_H(x_{<t}\oplus a)>0$, we may define
\begin{align*}
c=\frac12\min_{a\in\mathcal A_t}\phi_H(x_{<t}\oplus a)>0, \quad C=2\max_{a\in\mathcal A_t}\phi_H(x_{<t}\oplus a)<\infty,
\end{align*}
and let $E_M=\left\{c\leq \phi_{M,H}(x_{<t}\oplus a)\leq C,\ \forall a\in\mathcal A_t\right\}$. Because for every $a\in\mathcal A_t$ we have $\phi_{M,H}(x_{<t}\oplus a)\overset{p}{\longrightarrow}\phi_H(x_{<t}\oplus a)$ as $M\to\infty$, and $\mathcal A_t$ is finite, this convergence is uniform over $a\in\mathcal A_t$. Hence, $\Pr(E_M^c)\longrightarrow 0$. Define
\begin{align*}
U_{N,M}(A)=\frac1N\sum_{i=1}^N \frac{q_{t,M,H}(x_t^i)}{p_t(x_t^i)}\1\{x_t^i\in A\}, \qquad V_{N,M}=\frac1N\sum_{j=1}^N \frac{q_{t,M,H}(x_t^j)}{p_t(x_t^j)}.
\end{align*}
Then $R_{N,M}(A)=U_{N,M}(A)/V_{N,M}$. Moreover, conditional on $\phi_{M,H}$,
\begin{align*}
\E\left[U_{N,M}(A)\mid \phi_{M,H}\right] &= \E\left[\frac{q_{t,M,H}(x_t)}{p_t(x_t)}\1\{x_t\in A\}|\phi_{M,H}\right] = \sum_{a\in A} q_{t,M,H}(a) = q_{t,M,H}(A),\\
\E\left[V_{N,M}\mid \phi_{M,H}\right] &= \E\left[\frac{q_{t,M,H}(x_t)}{p_t(x_t)}|\phi_{M,H}\right] = \sum_{a\in\mathcal A_t} q_{t,M,H}(a) = 1.
\end{align*}
On the event $E_M$, for every $a\in\mathcal A_t$, it follows from \eqref{eq:densityMH} that
\begin{align*}
\frac{c}{C}\leq \frac{q_{t,M,H}(a)}{p_t(a)} = \frac{\phi_{M,H}(x_{<t}\oplus a)}{\sum_{a'\in\mathcal A_t} p_t(a')\phi_{M,H}(x_{<t}\oplus a')} \leq \frac{C}{c}.
\end{align*}
Hence, conditional on $\phi_{M,H}$ and $E_M$, both random variables $\frac{q_{t,M,H}(x_t^i)}{p_t(x_t^i)}\1\{x_t^i\in A\}$ and $\frac{q_{t,M,H}(x_t^i)}{p_t(x_t^i)}$ are bounded by $C/c$. Therefore, by Hoeffding's inequality, for every $\delta>0$,
\begin{align*}
\Pr\left(\left|U_{N,M}(A)-q_{t,M,H}(A)\right|>\delta | \phi_{M,H},E_M\right)\leq 2\exp\left(-\frac{2N\delta^2 c^2}{C^2}\right),\\
\Pr\left(\left|V_{N,M}-1\right|>\delta | \phi_{M,H},E_M\right)\leq 2\exp\left(-\frac{2N\delta^2 c^2}{C^2}\right).
\end{align*}

For any $\epsilon > 0$, let $\delta_\epsilon=\min\{\frac12,\frac{\epsilon}{4}\}$. On the event $\{|U_{N,M}(A)-q_{t,M,H}(A)|\leq \delta_\epsilon,\ |V_{N,M}-1|\leq \delta_\epsilon\}$, we have $V_{N,M}\geq \frac{1}{2}$, and therefore
\begin{align*}
|R_{N,M}(A)-q_{t,M,H}(A)| &= \left|\frac{U_{N,M}(A)}{V_{N,M}}-q_{t,M,H}(A)\right| \leq \frac{|U_{N,M}(A)-q_{t,M,H}(A)|}{V_{N,M}} + q_{t,M,H}(A)\frac{|V_{N,M}-1|}{V_{N,M}}\\
 &\leq 2|U_{N,M}(A)-q_{t,M,H}(A)| + 2|V_{N,M}-1| \leq \epsilon.
\end{align*}
Thus,
\begin{align*}
\Pr\left(|R_{N,M}(A)-q_{t,M,H}(A)|>\epsilon\right)\leq 4\exp\left(-\frac{2N\delta_\epsilon^2 c^2}{C^2}\right)+\Pr(E_M^c).
\end{align*}
which implies $R_{N,M}(A)-q_{t,M,H}(A)\overset{p}{\longrightarrow}0$.
Because $0\leq R_{N,M}(A)\leq 1$ and $0\leq q_{t,M,H}(A)\leq 1$, it follows from dominate convergence theorem that $\E\left[\left|R_{N,M}(A)-q_{t,M,H}(A)\right|\right]\longrightarrow 0$.

Next, because $\phi_{M,H}(x_{<t}\oplus a)\overset{p}{\longrightarrow}\phi_H(x_{<t}\oplus a)$ uniformly over $a\in\mathcal A_t$, it follows from \eqref{eq:densityMH} that for every $A\subseteq\mathcal A_t$, $q_{t,M,H}(A)\overset{p}{\longrightarrow} q_{t,H}(A)$. Since $0\leq q_{t,M,H}(A)\leq 1$, boundedness again implies $\E\left[\left|q_{t,M,H}(A)-q_{t,H}(A)\right|\right]\longrightarrow 0$.
Therefore, as $N, M \to \infty$,
\begin{align*}
\left|\Pr(x_{t,N,M}\in A)-q_{t,H}(A)\right| &= \left|\E[R_{N,M}(A)]-q_{t,H}(A)\right| \\
&\leq \E\left[\left|R_{N,M}(A)-q_{t,M,H}(A)\right|\right] + \E\left[\left|q_{t,M,H}(A)-q_{t,H}(A)\right|\right]\\
&\longrightarrow 0.
\end{align*}
This implies $x_{t,N,M}\Rightarrow q_{t,H}(\cdot)$.

\section{Examples}\label{app:examples}
The following four responses represent the answers to the case study problem in Section \ref{sec:myopic_insuff} under different generation methods.

\begin{nicebox}{lightred}{1. Standard sampling}
\small

\begin{itemize}[leftmargin=1.3em,itemsep=0.25em]
    \item \textbf{Mathematical Model:} To solve this optimization problem for the bodybuilder's diet, we will construct a linear programming model with the objective of minimizing the total cost while satisfying the minimum requirements for protein, carbohydrates, and calories.
    
    \item \textbf{Decision Variables:} ``$x_i$: The quantity of the $i$th food chosen, where $i$ can be Steak, Tofu, Chicken, Broccoli, Rice, or Spinach.''
    
    \item \textbf{Non-negativity Constraints:} ``$x_i \geq 0,\ \forall i$ as we cannot choose a negative quantity of food.'' \badhl{[integer restriction missing]}
\end{itemize}

\textbf{\ldots}
\end{nicebox}

\begin{nicebox}{lightred}{2. Low-temperature sampling}
\small

\begin{itemize}[leftmargin=1.3em,itemsep=0.25em]
    \item \textbf{Mathematical Model:} To solve the dietitian's optimization problem, we will establish a linear programming model aimed at minimizing the cost of food while meeting the nutritional requirements of the bodybuilder.
    
    \item \textbf{Decision Variables:} ``$x_1, x_2, x_3, x_4, x_5, x_6$: Represent the quantities of Steak, Tofu, Chicken, Broccoli, Rice, and Spinach, respectively.''
    
    \item \textbf{Non-negativity Constraints:} ``$x_1, x_2, x_3, x_4, x_5, x_6 \geq 0$, as the quantities of food cannot be negative.'' \badhl{[integer restriction missing]}
\end{itemize}

\textbf{\ldots}
\end{nicebox}

\begin{nicebox}{lightgreen}{3. power reward sampling}
\small

\begin{itemize}[leftmargin=1.3em,itemsep=0.25em]
    \item \textbf{Mathematical Model:} To solve this optimization problem, we will create a linear programming model. The objective is to minimize the total cost while meeting the nutritional requirements.
    
    \item \textbf{Decision Variables:} ``$x_1, x_2, x_3, x_4, x_5, x_6$: The quantities of Steak, Tofu, Chicken, Broccoli, Rice, and Spinach, respectively.''
    
    \item \textbf{Non-negativity Constraints:} ``$x_1, x_2, x_3, x_4, x_5, x_6 \geq 0$ \goodhl{and should be integers} to ensure feasible integer solutions.''
\end{itemize}

\textbf{\ldots}
\end{nicebox}

\begin{nicebox}{lightgreen}{4. entropy reward sampling}
\small

\textbf{Answer excerpt}

\begin{itemize}[leftmargin=1.3em,itemsep=0.25em]
    \item \textbf{Mathematical Model:} To solve this diet optimization problem, we will establish a linear programming model. The objective is to minimize the total cost while meeting the nutritional requirements.
    
    \item \textbf{Decision Variables:} ``$x_1, x_2, x_3, x_4, x_5, x_6$: Represent the quantities of Steak, Tofu, Chicken, Broccoli, Rice, and Spinach, respectively.''
    
    \item \textbf{Non-negativity Constraints:} ``$x_1, x_2, x_3, x_4, x_5, x_6 \geq 0$ \goodhl{and should be integers} to ensure feasible integer solutions.''
\end{itemize}

\textbf{\ldots}
\end{nicebox}

\begin{nicebox}{lightorange}{Takeaway}
\small
The objective and nutritional constraints are largely consistent across all four answers. The crucial difference lies in the \key{variable domain}. In both power reward and entropy reward sampling, candidates without the integer restriction can still appear during generation. However, once the model omits the integrality requirement, the subsequent reasoning process is affected: since the semantically appropriate formulation should use integer-valued servings, the unconstrained formulation creates ambiguity and makes the continuation more uncertain. Our reward-based methods reduce the probability of selecting such candidates by favoring continuations with lower uncertainty and more coherent downstream reasoning, thereby improving the final modeling accuracy.
\end{nicebox}
\end{document}